\documentclass[english,a4paper,12pt]{article}

\usepackage{amsxtra, amsfonts, amssymb, amstext, amsmath, mathtools}
\usepackage{amsthm}
\usepackage{booktabs}
\usepackage{nicefrac}
\usepackage{xspace}
\usepackage{url}
\usepackage{graphics,color}
\usepackage[algo2e,ruled,vlined,linesnumbered]{algorithm2e}
\usepackage{wrapfig}
\usepackage{bbm}
\usepackage{mathtools}
\usepackage{hyperref}

\newtheorem{theorem}{Theorem}
\newtheorem{lemma}[theorem]{Lemma}
\newtheorem{corollary}[theorem]{Corollary}
\newtheorem{definition}[theorem]{Definition}

\newcommand{\oea}{\mbox{$(1 + 1)$~EA}\xspace}

\newcommand{\oplea}{\mbox{$(1+\lambda)$~EA}\xspace}

\newcommand{\mplea}{\mbox{$(\mu+\lambda)$~EA}\xspace}
\newcommand{\mclea}{\mbox{$(\mu,\lambda)$~EA}\xspace}
\newcommand{\oclea}{\mbox{$(1,\lambda)$~EA}\xspace}
\newcommand{\ollga}{\mbox{$(1+(\lambda,\lambda))$~GA}\xspace}

\newcommand{\HLB}{\textsc{HLB}\xspace}
\newcommand{\DLB}{\textsc{DLB}\xspace}
\newcommand{\OM}{\textsc{OneMax}\xspace}
\newcommand{\Needle}{\textsc{Needle}\xspace}
\newcommand{\onemax}{\OM}

\newcommand{\LO}{\textsc{Leading\-Ones}\xspace}

\newcommand{\leadingones}{\LO}
\newcommand{\lo}{\textsc{L\-O}\xspace}
\newcommand{\needle}{\textsc{Needle}\xspace}
\newcommand{\funnel}{\textsc{Funnel}\xspace}

\newcommand{\cliff}{\textsc{Cliff}\xspace}
\newcommand{\balance}{\textsc{Balance}\xspace}
\newcommand{\valley}{\textsc{Valley}\xspace}

\newcommand{\jump}{\textsc{Jump}\xspace}

\DeclareMathOperator{\rwd}{randomWhereDifferent}
\DeclareMathOperator{\sido}{switchIfDistanceOne}
\DeclarePairedDelimiter{\ceil}{\lceil}{\rceil}

\DeclareMathOperator{\Ima}{Im}

\newcommand{\R}{\ensuremath{\mathbb{R}}}

\newcommand{\N}{\ensuremath{\mathbb{N}}} 
\newcommand{\Z}{\ensuremath{\mathbb{Z}}}

\newcommand{\calF}{\ensuremath{\mathcal{F}}}

\newcommand{\Var}{\mathrm{Var}\xspace} 
\newcommand{\eps}{\varepsilon}

\newcommand{\assign}{\leftarrow}

\let\originalleft\left
\let\originalright\right
\renewcommand{\left}{\mathopen{}\mathclose\bgroup\originalleft}
\renewcommand{\right}{\aftergroup\egroup\originalright}

\date{}

\begin{document}

\title{Choosing the Right Algorithm With Hints From Complexity Theory\thanks{Extended version of a paper in the proceedings of IJCAI~2021~\cite{WangZD21}. This version, besides more details, contains all mathematical proofs. Corresponding authors: Weijie Zheng and Benjamin Doerr.}}
\author{Shouda Wang\\ Laboratoire d'Informatique (LIX)\\ \'Ecole Polytechnique, CNRS\\ Institut Polytechnique de Paris\\ Palaiseau, France
\and Weijie Zheng\\ Department of Computer Science and Engineering \\ Southern University of Science and Technology \\ Shenzhen, China
\and Benjamin Doerr \\ Laboratoire d'Informatique (LIX)\\ \'Ecole Polytechnique, CNRS\\ Institut Polytechnique de Paris\\ Palaiseau, France
}
\maketitle

\begin{abstract}
Choosing a suitable algorithm from the myriads of different search 
heuristics is difficult when faced with a novel optimization problem. 
In this work, we argue that the purely academic question of what could be the best possible algorithm in a certain broad class of black-box optimizers can give fruitful indications in which direction to search for good established optimization heuristics.
We demonstrate this approach on the recently proposed
DLB benchmark, for which the only known results are $O(n^3)$ runtimes 
for several classic evolutionary algorithms and an $O(n^2 \log n)$ 
runtime for an estimation-of-distribution algorithm. Our finding that the unary unbiased black-box complexity is only $O(n^2)$ suggests the Metropolis algorithm as an interesting candidate and we prove that 
it solves the DLB problem in quadratic time. Since we also prove that better runtimes cannot be obtained in the class of unary unbiased algorithms, we shift our attention to algorithms that use the information of more parents 
to generate new solutions. An artificial algorithm of this type having an 
$O(n \log n)$ runtime leads to the result that 
the significance-based compact genetic algorithm (sig-cGA) can solve the DLB 
problem also in time $O(n \log n)$ with high probability. 
Our experiments show a remarkably good performance of the Metropolis algorithm, clearly the best of all algorithms regarded for reasonable problem sizes.
\end{abstract}

%

%
%
%


{\sloppy
\section{Introduction}

Randomized search heuristics such as stochastic hillclimbers, evolutionary algorithms, ant colony optimizers, or estimation-of-distribution algorithms (EDAs) have been very successful at solving optimization problems for which no established problem-specific algorithm exists. As such, they are applied massively to novel problems for which some understanding of the problem and the desired solution exists, but little algorithmic expertise. 

When faced with a novel optimization problem, one has the choice between a large number of established heuristics (and an even larger number of recent metaphor-based heuristics~\cite{Sorensen15}). Which of them to use is a difficult question. Since implementing a heuristic and adjusting it to the problem to be solved can be very time-consuming, ideally one does not want to experiment with too many different heuristics. For that reason, a theory-guided prior suggestion could be very helpful. This is what we aim at in this work. We note that the theory of randomized search heuristics has helped to understand these algorithms, has given suggestions for parameter settings, and has even proposed new operators and algorithms (see the textbooks~\cite{NeumannW10,AugerD11,Jansen13,ZhouYQ19,DoerrN20} or the tutorial~\cite{Doerr20gentle}), but we are not aware of direct attempts to aid the initial choice of the basic algorithm to be used (as with experimental work, there always is the indirect approach to study the existing results and try to distill from them some general rule which algorithms perform well on which problems, but in particular for the theory domain it is not clear how effective this approach is at the moment).

What we propose in this work is a heuristic approach building on the notion of black-box complexity, first introduced by Droste, Jansen, Tinnefeld, and Wegener~\cite{DrosteJTW02} (journal version~\cite{DrosteJW06}). In very simple words, the (unrestricted) black-box complexity of an optimization problem is the performance of the best black-box optimizer for this problem. It is thus a notion not referring to a particular class of search heuristics such as genetic algorithms or EDAs. Black-box complexity has been used successfully to obtain universal lower bounds. Knowing that the black-box complexity of the \Needle problem is exponential~\cite{DrosteJW06}, we immediately know that no genetic algorithm, ant colony optimizer, or EDA can solve the Needle problem in subexponential time. With specialized notions of black-box complexity, more specific lower bounds can be obtained. The result that the unary unbiased black-box complexity of the \onemax benchmark is at least of the order $n \log n$~\cite{LehreW12} implies that many standard mutation-based evolutionary algorithms cannot optimize \onemax faster than this bound. 

With a more positive perspective, black-box complexity has been used to invent new algorithms. Noting that the unary unbiased black-box complexity of \onemax is $\Omega(n \log n)$, but the two-ary (i.e., allowing variation operators taking two parents as input) unbiased black-box complexity is only $O(n)$~\cite{DoerrJKLWW11}, a novel crossover-based evolutionary algorithm was developed in~\cite{DoerrDE15}. Building on the result that the unary unbiased $\lambda$-parallel black-box complexity of the \onemax problem is only $O(\frac{n\lambda}{\log \lambda} + n \log n)$~\cite{BadkobehLS14,LehreS20}, dynamic, self-adjusting, and self-adapting EAs obtaining this runtime have been constructed~\cite{BadkobehLS14,DoerrGWY19,DoerrWY21}.

In this work, we also aim at profiting from the guidance of black-box results, however not by designing new algorithms, but by giving an indication which of the existing algorithms could be useful for a particular problem. Compared to the approach taken in~\cite{DoerrDE15}, we can thus profit from the numerous established and well-understood algorithms and avoid the risky and time-consuming road of developing a new algorithm.

In simple words, what we propose is trying to find out which classes of black-box algorithms contain fast algorithms for the problem at hand. These algorithms may well be artificial as we use them only to determine the direction in which to search for a good established algorithm for our problem. Only once we are sufficiently optimistic that a certain property of black-box algorithms is helpful for our problem, we regard the established heuristics in this class and see if one of them indeed has a good performance.

To show that this abstract heuristic approach towards selecting a good algorithm can indeed be successful, we regard the DeceivingLeadingBlocks (DLB) problem recently proposed by Lehre and Nguyen~\cite{LehreN19foga}. Lehre and Nguyen conducted rigorous runtime analyses of several classic evolutionary algorithms, all leading to runtime guarantees of order $O(n^3)$ with optimal parameter choices. For the EDA \emph{univariate marginal distribution algorithm (UMDA)}, a runtime guarantee of $O(n^2 \log n)$ was proven in~\cite{DoerrK21ecj}. No other proven upper bounds on runtimes of randomized search heuristics on the \DLB problem existed prior to this work. With only these results from only two prior works, it is safe to call the DLB problem relatively novel and thus an interesting object for our investigation. 

\emph{Finding more efficient randomized search heuristics:} 
We note that the classic algorithms regarded in~\cite{LehreN19foga} are all elitist evolutionary algorithms or non-elitist algorithms with parameter settings that let them imitate an elitist behavior. This choice was natural given the weaker understanding of non-elitist algorithms and the fact that not many theoretical works could show a convincing advantage of non-elitist algorithms (see Section~\ref{sssec:nonel}). To obtain a first indication whether it is worth investigating non-elitist algorithms for this problem, we show two results. (i)~We prove that the $(1+1)$ elitist unbiased black-box complexity of the DLB problem is $\Omega(n^3)$. 
(ii)~We show that there is a simple, artificial, $(1+1)$-type non-elitist unbiased black-box algorithm solving the DLB problem in quadratic time. These two findings motivate us to analyze the existing ${(1+1)}$-type non-elitist heuristics. Among them, we find that the Metropolis algorithm~\cite{Metropolis} with a suitable temperature also optimizes DLB in time $O(n^2)$. We note that there are very few runtime analyses on the Metropolis algorithm (see Section~\ref{sssec:metro}), so it is clear that a synopsis of the existing runtime analysis literature would not have easily suggested this algorithm. 

To direct our search for possible further runtime improvements, we first show that the unary unbiased black-box complexity of DLB is at least quadratic. Consequently, if we want to stay in the realm of unbiased algorithms (which we do) and improve beyond quadratic runtimes, then we necessarily have to regard algorithms that are not unary, that is, that generate offspring using the information from at least two parents. That this is possible, at least in principle, follows from our next result, which is an artificial crossover-based algorithm that solves DLB in time $O(n \log n)$. While, together with the previously shown lower bound, it is clear that this performance relies on the fact that offspring are generated from the information of more than one parent, the working principles of this algorithm also include a learning aspect. The algorithm optimizes the blocks of the DLB problem in a left-to-right fashion, but once a block is optimized, it is never touched again. Such learning mechanisms are rarely found in standard evolutionary algorithms, but are the heart of EDAs with their main goal of learning a distribution that allows sampling good solutions. Note that the distribution in an EDA carries information from many previously generated solutions, hence EDAs necessarily generate new solutions based on the information of many parents. For these reasons, we focus on EDAs. We do not find a classic EDA  for which we can prove that it solves the DLB problem in subquadratic time, but we succeed for the significance-based EDA~\cite{DoerrK20tec} and we show that it optimizes the DLB problem in a runtime of~ $O(n \log n)$ with high probability. 

Overall, these results demonstrate that our heuristic theory-guided approach towards selecting good algorithms for a novel problem can indeed be helpful. We note in particular that the previous works on the DLB problem have not detected that the Metropolis algorithm is an interesting candidate for solving this problem. Our experimental analysis confirms a very good performance of the Metropolis algorithm, but suggests that the runtimes of the EDAs suffer from large constants hidden by the asymptotic analysis.

To avoid a possible misunderstanding, let us stress that our target is to find an established search heuristic for our optimization problem. From the above discourse one could believe that we should simply stick to the artificial black-box algorithm that we found. If our only aim was solving the DLB problem, this would indeed be feasible. Such an algorithm, however, would most likely lack the known positive properties of established search heuristics such as robustness to noise and dynamic changes of the problem data, reusability for similar problems, and adjustability to restricted instance classes. For that reason, our target in this work is definitely to find an established heuristic and not an particular algorithm custom-tailored to a problem.

We note that the results and methods used in this works lie purely in the theory domain. We therefore followed the traditional approach~\cite{DrosteJW02} of regarding benchmark problems simple enough that they can be rigorously analyzed with mathematical means. In return, we obtain proven results for infinite numbers of problem instances (here, the DLB problem for all problem dimensions $n \in 2\N$), which hopefully extend in spirit also to problems which are too complicated to be analyzed with mathematical means. 

We believe that our approach, in principle and in a less rigorous way, can also be followed by researchers and practitioners outside the theory community. Our basic approach of (i)~trying to find a very good algorithm, chosen from all possible black-box optimization algorithms, to solve a given problem or to overcome a particular difficulty and then (ii)~using this artificial and problem-specific algorithm as indicator for which established search heuristics could be promising, can also be followed by experimental methods and by non-rigorous intuitive considerations. 

The remainder of the paper is organized as follows. Section~\ref{sec:pre} discusses black-box optimization, black-box complexity, the \DLB problem, and two probabilistic tools to be used later.
Sections~\ref{sec1} and~\ref{sec:beyond} demonstrate our approach on the \DLB problem. We first observe that non-elitist (1+1) type algorithms in principle can improve the $O(n^3)$ runtime of known elitist algorithms to a quadratic runtime and find the Metropolis algorithm as an established heuristic showing this performance. By going beyond unary unbiased algorithms, we then obtain a further improvement to a complexity of $O(n \log n)$, first via an artificial algorithm and then via the significance-based compact genetic algorithm. Our experimental discussion is shown in Section~\ref{sec:exp}, and Section~\ref{sec:con} concludes this paper.

\section{Preliminaries}
 \label{sec:pre}
Following the standard notation,
we write 
$[\ell . . m]:=\{\ell,\ell+1,\dots,m\}$ for all $\ell,m\in \mathbb{N}$ such that $\ell \le m$. 
In this paper we consider pseudo-Boolean optimization problems,
that is,
 problems of \emph{maximizing} functions $f $ defined on the search space $\{0,1\}^n$, where $n$ is a positive integer.
A bit string (also called a search point) is an element of the set $\{0,1\}^n$ and is represented as $x=(x_1,\dots,x_n)$.
For any set $S=\{s_1,\dots,s_{|S|}\}\subseteq [1..n]$ with $s_i < s_j$ for $i < j$, we write $x_S:= (x_{s_1},\dots,x_{s_{|S|}})$.
In the context of evolutionary algorithms, we sometimes refer to a bit string as an individual and we   use $x^{(t)}$ to denote an individual at time $t$,
in which case $x^{(t)}_i$ is used to represent the $i$-th component of the individual $x^{(t)}$.
To simplify the notation,
whenever we write $m$ where an integer is required, 
we implicitly mean $\ceil{m}\vcentcolon = \min \{k\in \mathbb{N}\mid k\geq m\}$.

\subsection{Black-Box Optimization and Runtime}

In (discrete) black-box optimization, we assume that the optimization algorithms do not have access to an explicit description of the instance of the problem to be solved. Instead, their only access to the instance is via a black-box evaluation of search points. Classic black-box optimization algorithms include hill-climbers, the Metropolis algorithm, simulated annealing, evolutionary algorithms, and other bio-inspired search heuristics. 

A general scheme for a black-box algorithm $A$ is given in Algorithm~\ref{bbal}. It starts by generating a random search point according to a given probability distribution and evaluating it. It then repeats generating (and evaluating) new search points based on all information collected so far, that is, based on all previous search points together with their fitnesses. While in practice, naturally, this iterative procedure is stopped at some time, in theoretical investigations it is usually assumed that this loop is continued forever. In this case, the \emph{runtime}, also called \emph{optimization time}, $T := T_A(f)$ of the algorithm $A$ on the problem instance $f : \{0,1\}^n \to \R$ is the number of search points evaluated until (and including) for the first time an optimal solution of $f$ is generated (and evaluated). In the notation of Algorithm~\ref{bbal}, we have~ $T_{A}(f) = 1+\inf\{t\mid x^{(t)}\in\arg\max f\}$. If the algorithm $A$ is randomized (which most black-box optimizers are), then the runtime $T$ is a random variable and often only its expectation $E[T]$ is analyzed. For an \emph{optimization problem}, that is, a set $\calF$ of problem instances $f$, the \emph{worst-case expected optimization time} $\sup_{f \in \calF} E[T_A(f)]$ is an often regarded quality measure of the algorithm $A$. 

Most black-box algorithms for practical reasons do not exploit this scheme to its full generality, that is, they do not store all information collected during the process or they do not sequentially generate one search point after the other based on all previous information, but instead generate in parallel several search points based on the same set of information. For our analyses, we nevertheless assume that the search points are generated in some specified order and that \emph{each search point is evaluated immediately after being generated}. 

\begin{algorithm2e}%
Generate a search point $x^{(0)}$ according to a given distribution $\mathcal{D}^{(0)}$ on $\{0,1\}^n$;


\For{$t=1,2,3,\dots$}{
Depending on $f(x^{(0)}),\dots,f(x^{(t-1)})$ and $x^{(0)},\dots,x^{(t-1)}$, 
choose a probability  distribution $\mathcal{D}^{(t)}$ on $\{0,1\}^n$ \label{line}; 

Sample $x^{(t)}$ from $\mathcal{D}^{(t)}$;

   }
\caption{Template of a  black-box algorithm for optimizing an unknown function $f: \{0,1\}^n \rightarrow \mathbb{R}$. Without explicit mention, we assume that each search point $x^{(t)}$ is evaluated immediately after being generated. As runtime of such an (infinitely running) algorithm we declare the number of search points evaluated until an optimum of $f$ is evaluated for the first time. }
\label{bbal}
\end{algorithm2e}

\subsection{Unbiasedness}
\label{subsec:unbiase}
 
Unless a specific understanding of the problem at hand suggests a different approach, it is natural to look for optimization algorithms that are invariant under the symmetries of the search space. This appears so natural that, in fact, most algorithms have this property without that this has been discussed in detail. 

The first to explicitly discuss such invariance properties for the search space $\{0,1\}^n$ of bit strings (see, e.g., \cite{TeytaudGM06} for such a discussion for continuous search spaces) were Lehre and Witt in their seminal paper~\cite{LehreW12}. They coined the name \emph{unbiased} for algorithms respecting the symmetries of the hypercube $\{0,1\}^n$. Such algorithms treat the bit positions $i \in [1..n]$ in a symmetric fashion and, for each bit position, do not treat the value $0$ differently from the value $1$. It follows that all decisions of such algorithms may not depend on the particular bit string representation of the search points they have generated before. Rather, all decisions can only depend on the fitnesses of the search points generated. This implies that all search points the algorithm has access to can only be generated from previous ones via unbiased variation operators. This observation also allows to rigorously define the arity of an algorithm as the maximum number of parents used to generate offspring. Hence mutation-based algorithms have an arity of one and crossover-based algorithms have an arity of two. We note that sampling a random search point is an unbiased operator with arity zero. 

Since these are important notions that also help to classify different kinds of black-box algorithms, let us make them precise in the following. 

\begin{definition}\label{ununopt}
A $k$-ary variation operator $V$ is a function that assigns to each $k$-tuple of bit strings in $\{0,1\}^n$ a probability distribution on $\{0,1\}^n$. It is called \emph{unbiased} if
\begin{itemize}
\item $\forall x^1,\dots,x^k \in \{0,1\}^n, \forall y, z \in \{0,1\}^n,\\
\quad \Pr[y =V( x^{1}, \ldots, x^{k})]=\Pr[y \oplus z =V(x^{1} \oplus z, \ldots, x^{k} \oplus z)]$,
\item $\forall x^1,\dots,x^k \in \{0,1\}^n, \forall   y\in \{0,1\}^n, \forall\sigma\in \mathcal{S}_n,\\
\quad \Pr[y =V( x^{1}, \ldots, x^{k})]=\Pr[\sigma(y) =V( \sigma(x^{1}), \ldots, \sigma(x^{k}))]$,
\end{itemize}
where $\oplus$ represents the exclusive-or operator and $\mathcal{S}_n$ represents the symmetric group on $n$ letters and we write $\sigma(x)=(x_{\sigma(1)},\dots,x_{\sigma(n)})$ for all $\sigma\in \mathcal{S}_n$ and $x\in \{0,1\}^n$.
\end{definition}

By definition, a $k$-ary operator can simulate $\ell$-ary operators if $\ell\leq k$.
It is also immediate that the only $0$-ary operator is the operator that generates a search point in $\{0,1\}^n$ uniformly at random.
As a special case,
$1$-ary unbiased variation operators are more often called unary unbiased  variation operators and sometimes referred  to as mutation operators in the context of evolutionary computation.

Unary unbiased variation operators admit a simple characterization, namely that sampling from the unary unbiased operator is equivalent to sampling a number $r \in [0..n]$ from some distribution and then flipping exactly $r$ bits chosen uniformly at random. This is made precise in the following lemma, which was proven in \cite[Lemma~1]{DoerrDY20}, but which can, in a more general form, already be found in~\cite{DoerrKLW13tcs}.
\begin{lemma}\label{charc_uu}
Let $D$ be a probability distribution on $[0..n]$. Let $V_D$ be the unary variation operator which for each $x \in \{0,1\}^n$ generates $V_D(x)$ by first sampling a number $r$ from $D$ and then flipping $r$ random bits in $x$. Then $V_D$ is a unary unbiased variation operator. 

Conversely, let $V$ be a unary unbiased variation operator on $\{0,1\}^n$. Then there is a probability distribution $D_V$ on $[0..n]$ such that $V = V_{D_V}$.
\end{lemma}

Building on the notion of a $k$-ary unbiased variation operator, we can now define what is a $k$-ary unbiased black-box algorithm (Algorithm~\ref{alg:unbiased}). For the reasons given at the beginning of this section, in this work we shall only be interested in unbiased algorithms (possibly with unrestricted arity).
\begin{algorithm2e}
Generate $x^{(0)}$   uniformly at random\;
\For{$t=1,2,3,\dots$}{
   Based solely on $f(x^{(0)}), \dots, f(x^{(t-1)})$, choose a $k$-ary unbiased variation operator $V$ and $i_1, \dots, i_k \in [0..t-1]$\;	
	 Sample $x^{(t)}$ from $V(x^{(i_1)}, \dots, x^{(i_k)})$\;
   }
\caption{Template of a $k$-ary unbiased black-box algorithm for optimizing an unknown function $f: \{0,1\}^n\rightarrow \mathbb{R}$. }
\label{alg:unbiased}
\end{algorithm2e}

\subsection{Black-Box Complexity}

To understand the difficulty of a problem for black-box optimization algorithms, inspired by classical complexity theory, Droste, Jansen, and Wegener~\cite{DrosteJW06} (preliminary version~\cite{DrosteJTW02}, see also~\cite{Doerr20BBbookchapter} for a recent survey) defined the \emph{black-box complexity} of a problem $\calF$ as the smallest worst-case expected runtime a black-box algorithm $A$ can have on $\calF$, that is, 
\[
\inf_{A}\sup_{f\in \mathcal{F}}E[T_{A}(f)],
\]
where $A$ runs over all black-box algorithms in the infimum.

Just by definition, the black-box complexity is a universal lower bound on the performance of any black-box algorithm. The result that black-box complexity of the \needle problem is $2^{n-1} + 0.5$ \cite[Theorem~1]{DrosteJW06} immediately implies that no hillclimber, evolutionary algorithm, ant colony optimizer, etc. can solve the \needle problem in subexponential expected time. 

Conversely, the black-box complexity can also serve as a trigger to search for better black-box algorithms. For example, in~\cite{DoerrDE15} the observation that the black-box complexity of the \onemax problem is only $\Theta(n / \log n)$ \cite{ErdosR63,DrosteJW06,AnilW09} (albeit witnessed by a highly artificial algorithm) whereas most classic evolutionary algorithms have a $\Theta(n \log n)$ runtime or worse, was taken as starting point to design a novel evolutionary algorithm solving the \onemax problem in time asymptotically better than $n \log n$. This algorithm, called \ollga, has shown a good performance both in experiments~\cite{GoldmanP14,MironovichB17,BuzdalovD17} and in other mathematical analyses~\cite{BuzdalovD17,DoerrD18,AntipovDK19foga,AntipovBD20ppsn,AntipovD20ppsn,AntipovBD21gecco,AntipovDK22,AntipovBD22}. 

In this work, as discussed in the introduction, we shall also use the black-box complexity as a trigger towards more efficient solutions to black-box optimization problems, however, not by suggesting new algorithms, but by suggesting the general type of algorithm that might be most suited for the problem and thus helping to choose the right algorithm among the many existing black-box algorithms. Compared to triggering the design of new algorithms, this might be the more effective road for people ``merely'' applying black-box algorithms. In fact, we shall argue that this road, despite the theory-based notion of black-box complexity, is in fact not too difficult to follow also for people without a background in algorithm theory.

\subsection{The DLB Function and Known Runtimes}\label{sec:dlb}

We now define the $\DLB$ function, which is the main object of our study and was first introduced  by Lehre and Nguyen
in  their recent work \cite{LehreN19foga}.

To define the \DLB function, the $n$-bit string $x$ is divided, in a left-to-right fashion, into $\tfrac n2$ blocks of size of $2$. The function value of $x$ is determined by the longest prefix of $11$ blocks and the following block. The blocks with two $1$s in the prefix contribute each a value of $2$ to the fitness. The \DLB function is deceptive in that the next block contributes a value of $1$ when it contains two $0$s, but contributes only $0$ when it contains one $1$ and one $0$. The optimum is the bit string with all $1$s.

The $\DLB$ function is defined on $\{0,1\}^n$ only for $n$ even.
We therefore assume in the following that $n$ is even whenever this is required.

For all  $x\in \{0,1\}^n$, we formally define the $\DLB$ function in the following way. 
We consider blocks of the form $(x_{2\ell+1},x_{2\ell+2})$.
If $x\neq (1,\dots,1)$, then let $(x_{2m+1},x_{2m+2})$ be the first block that is not a $11$ block, that is, $m=\inf\{\ell\mid x_{2\ell+1}\neq 1 \text{ or }x_{2\ell+2}\neq 1\}$.
We call such a block a \emph{critical block}.
Then the $\DLB$ function is defined via
$$
  \DLB(x)=
  \begin{cases}
  2m+1 &\text{if } x_{2m+1}+x_{2m+2}=0,\\
  2m &\text{if } x_{2m+1}+x_{2m+2}=1,\\
  n &\text{if } x=(1,\dots,1).
  \end{cases}
$$

In other words, the $\DLB$ function counts twice the number of consecutive 11 blocks until it reaches 
a critical block, which counts for 1 if it is of the form 00
and  counts for 0 if it is of the form 01 or 10.
Hence the  search points $x=(1,\dots,1,0,0,x_{2\ell+1},\dots,x_n)$  with $\ell\in[1..\frac{n}{2}]$ and $x_{2\ell+1},\dots,x_n\in\{0,1\}$
are the local maxima of the $\DLB$ function.
The unique global maximum of the $\DLB$ function is   $x^*=(1,\dots,1)$.


In a similar fashion we define the  \emph{Honest Leading Blocks} ($\HLB$) function, 
which will be used as potential function in some proofs using drift analysis.  
The $\DLB$ function being deceptive and unable to discern a $01$ critical block from a $10$ critical block, 
  the $\HLB $ function  also treats $01$ and $10$ critical blocks equally,
but is honest with the fact that $01$ and $10$ critical blocks are better than a $00$ critical  block in the sense that such search points  are closer to the global maximum $(1,\dots,1)$.
Formally speaking, the $\HLB$ function with parameter $\delta\in(0,2)$ is defined by 
\begin{equation*}
\HLB_{\delta}(x)\vcentcolon=\left\{\begin{array}{ll} 
2  m & \text { if } \DLB(x)=2m+1, \\ 
2  m+2-\delta & \text { if } \DLB(x)=2m, \\
n & \text { if } \mathrm{DLB}(x)=n,\end{array}\right.
\end{equation*}
where  $m$ is an integer  in $\{0,1,\dots, \frac{n}{2}-1\}$.

We now review the most relevant known runtime results for this work. 
Lehre and Nguyen   \cite{LehreN19foga} proved that, always assuming that the mutation rate used by the algorithm is $\chi/n$ for a constant $\chi > 0$,  the expected runtime of the $(1+\lambda)$ EA on the $\DLB$ problem
is $O(n\lambda+n^3)$
and that  the expected runtime of the $(\mu+1)$ EA on the $\DLB$ problem is $O(\mu n \log n +n^3)$.
They also proved that the  expected runtime of the $(\mu,\lambda)$ EA on the $\DLB$ function is $O(n\lambda\log \lambda + n^3)$ under the conditions that for an arbitrary constant $\delta > 0$ and a constant $c$ sufficiently large, we have $\lambda > c \log n $ and $\mu< \frac{\lambda e^{-2\chi}}{1+\delta}$.
Furthermore, they showed that, with a good choice of the parameters, 
genetic algorithms using $k$-tournament selection,  $(\mu,\lambda)$ selection,  linear selection, or exponential ranking selection,
also have an $O(n\lambda\log \lambda + n^3)$ expected runtime on the $\DLB$ problem. 

From looking at the proofs in~\cite{LehreN19foga}, it appears natural that all algorithms given above have a runtime of at least $\Omega(n^3)$ on the \DLB problem, but the only proven such result is that Doerr and Krejca~\cite{DoerrK21ecj} showed that the $(1+1)$ EA with mutation rate $1/n$ solves the $\DLB$ problem in $\Theta(n^3)$ expected fitness evaluations. In Theorem~\ref{1+lam}, we extend this result to all $(1+1)$-elitist unary unbiased black-box algorithms.

As opposed to these polynomial  runtime results, 
Lehre and Nguyen   pointed out in 
\cite{LehreN19foga} a potential weakness of the Univariate Marginal Distribution Algorithm (UMDA).
They proved that the UMDA selecting $\mu$ fittest individuals from $\lambda$ sampled individuals has an expected runtime of~ $e^{\Omega(\mu)}$ on the $\DLB$ problem if $\frac{\mu}{\lambda}\geq \frac{14}{1000}$ and  $c \log n\leq \mu=o(n)$ for some sufficiently large constant $c>0$,
and has expected runtime  $O(n\lambda\log \lambda + n^3)$  if~ $\lambda \geq(1+\delta) e \mu^{2}$ for any $\delta>0$.
However, Doerr and Krejca~\cite{DoerrK21ecj} pointed out that this negative finding can be overcome with a different parameter choice and that with a population size large enough to prevent genetic drift \cite{SudholtW19,DoerrZ20tec}, 
the UMDA solves the $\DLB$ problem efficiently.
To be precise, they proved that the UMDA optimizes the $\DLB$ problem within   probability at least $1-\frac{9}{n}$  within $\lambda\left(\frac{n}{2}+2 e \ln n\right)$ fitness evaluations if $\mu \geq c_{\mu} n \ln n$ and~ $\mu / \lambda \leq c_{\lambda}$ for some $c_\mu$, $c_\lambda$ sufficiently large or small, respectively.


Since there is an apparent similarity between the $\DLB$ function and the classic $\LO$ benchmark function, we recall the definition of the latter.
The \LO function is defined for all $x\in\{0,1\}^n$ by $\sum_{r\in [1..n]} \prod_{s\in [1..r]}x_s$. 
Right from the definitions, we see that the functions \DLB and \LO are very similar, indeed, we have $| \LO(x)-\DLB(x)|\leq 1$ for all $x\in \{0,1\}^n$. The main difference is that \DLB has non-trivial local optima which could render its optimization harder. 

Due to the similarity between the \DLB and \leadingones problems, it will be useful to compare the runtimes on these two problems. The \leadingones problem was proposed in~\cite{Rudolph97} as an example of a unimodal problem which simple EAs cannot solve in $O(n \log n)$ time. The correct asymptotic runtime of the \oea of order $\Theta(n^2)$ was determined in~\cite{DrosteJW02}.
The precise runtime of the \oea was independently determined in~\cite{BottcherDN10,Sudholt13}. Precise runtimes of various other variants of the \oea were given in \cite{Doerr19tcs}.
The runtime of the \mplea with~ $\mu$ at most polynomial in $n$ on \leadingones is $\Theta(\mu n \log n + n^2)$~\cite{Witt06}, the one of the \oplea with $\lambda$ at most polynomial in $n$ is $\Theta(\lambda n + n^2)$. 
For the \mplea, only a lower bound of $\Omega(\frac{\lambda n}{\log(\lambda/n)})$ is known~\cite{BadkobehLS14,LehreS20}. The runtime of the \ollga with standard parameterization $p = \lambda/n$ and~ $c = 1/\lambda$ is $\Theta(n^2)$ regardless of the value of $\lambda \in [1..n/2]$ and this also with dynamic parameter choices~\cite{AntipovDK19foga}. Consequently, for a large number of elitist algorithms, the runtime is  $\Theta(n^2)$ when suitable parameters are used. 
In \cite{DoerrL18}, 
Doerr and Lengler have shown that all $(1+1)$-elitist   algorithms need~ $\Omega(n^2)$ fitness evaluations in expectation to solve the \LO problem.
This result implies the lower bounds in \cite{DrosteJW02,BottcherDN10,Sudholt13,Doerr19tcs} when ignoring constant factors.

For non-elitist algorithms, the picture is less clear. Upper bounds have been shown for various algorithms, also some using crossover, when the selection pressure is high enough~\cite{Lehre11,DangL16algo,CorusDEL18,DoerrK21algo}, but none of them beats the $\Omega(n^2)$ barrier. When the selection pressure is small, many non-elitist algorithm cannot optimize any function with unique optimum in subexponential time~\cite{Lehre10,Doerr21ecjLB}.

Upper bounds were also shown for the runtime of the estimation-of-distribution algorithms UMDA and PBIL in~\cite{DangLN19,LehreN21} and for the ant-colony optimizers 1-ANT and MMAS in~\cite{DoerrNSW11,NeumannSW09}, but again none could show a runtime better than quadratic in $n$. 

A runtime better than quadratic, and in fact of order $O(n \log n)$ was shown for the three non-classical algorithms CSA~\cite{MoraglioS17}, scGA~\cite{FriedrichKK16}, and sig-cGA~\cite{DoerrK20tec}. The first two of these, however, are highly inefficient on the \onemax benchmark and thus might be overfitted to the \leadingones problem.

The unrestricted black-box complexity of the \leadingones class is $\Theta(n \log\log n)$, as witnessed by a highly problem-specific algorithm in~\cite{AfshaniADDLM19}.

\subsection{Probabilistic Tools}

We now collect two probabilistic tools used in the remainder of this work. 
The additive drift theorem is commonly used to derive upper (resp. lower) bounds on the expected  runtime of an 
algorithm from lower (resp. upper) bounds on the  expected increase of a suitable  potential.
It first appeared in the analysis of evolutionary algorithms   in He and Yao's work \cite{HeY01,HeY04}, in which they implicitly used the optional stopping theorem for martingales.
The following version can be found in  Lengler's survey~\cite{Lengler20bookchapter}.
It was first proven in Lengler and  Steger's work \cite{LenglerS18} via an approach different from the one  He and Yao used. 
\begin{theorem}[\cite{Lengler20bookchapter}, Theorem 2.3.1]\label{adddrift}
Let $(h_t)_{t\geq 0}$ be a sequence of non-negative random variables 
taking values in a finite set $\mathcal{S}\subseteq [0,n]$
such that  $n\in \mathcal{S}$.
Let 
$\ T:=\inf \left\{t \geq 0 \mid h_{t}=n\right\}$ be the first time when $(h_t)_{t\geq 0}$ takes the value  $n$. 
 For all $ s\in\mathcal{S}$, let $
\Delta_{t}(s):=E\left[h_{t+1}- h_{t}\mid h_{t}=s\right]$.
Then the following two assertions hold.
\begin{itemize}
    \item  If for some $\delta > 0$
 we have $\Delta_{t}(s)\geq \delta$ for all $s\in \mathcal{S}\setminus \{n\}$ and all $t$,
then $E[T]\leq  E[n-h_0] \delta^{-1}$.
 \item  If for some $\delta> 0$ we have 
$\Delta_{t}(s)\leq \delta$ for all $s\in \mathcal{S}\setminus \{n\}$ and all $t$,
then $E[T]\geq  E[n-h_0] \delta^{-1}$.

\end{itemize}

\end{theorem}

Chernoff-Hoeffding inequalities~\cite{Chernoff52,Hoeffding63}, often just called \emph{Chernoff bounds}, are a standard tool in the analysis of random structures and algorithms. In this work we will use the following variance-based additive Chernoff bound (see, e.g.,~\cite[Theorem~1.10.12]{Doerr20bookchapter}).
\begin{theorem}\label{chern}(variance-based additive Chernoff inequality)
 Let $X_1,\dots,X_n$ be independent random variables and suppose that  for all $i\in[1..n]$, we have $|X_i-E[X_i]|\leq 1$. Let $X\vcentcolon=\sum_{i=1}^nX_i$ and $\sigma^2\vcentcolon=\Var[X]=\sum_{i=1}^n \Var[X_i]$. Then for all $\lambda\in(0,\sigma^2)$,
 $$
 \Pr[X \geq E[X]+\lambda] \leq e^{-  \frac{\lambda^2}{3\sigma^2}}
 \text{ and }
 \Pr[X \leq E[X]-\lambda] \leq e^{-  \frac{\lambda^2}{3\sigma^2}}.
 $$
\end{theorem}

\section{From Elitist to Non-Elitist Algorithms} \label{sec1}

Previous works have shown that the expected runtime of the $(1+1)$
EA on the $\DLB$ problem is $\Theta(n^3)$, see \cite[Theorem~3.1]{LehreN19foga} for the upper and \cite[Theorem~4]{DoerrK21ecj} for the lower bound (following from the precise computation of the expected runtime there).

In this section, we extend this lower bound and show that any $(1+1)$-elitist unary unbiased black-box algorithm has a runtime of at least $\Omega(n^3)$. This result  will motivate us to study non-elitist $(1+1)$-type algorithms, which will lead to the discovery that the Metropolis algorithm can solve the \DLB problem significantly faster. 

\subsection{Elitist Algorithms Suffer From the Fitness Valleys} 

\subsubsection{$(\mu + \lambda)$-Elitist Black-Box Complexity}
We start by making precise the elitist black-box complexity model we regard. Since it might be useful in the future, we first define our model for general $(\mu+\lambda)$-elitism, even though our main result in this section only considers $(1+1)$-elitist algorithms. 

A $(\mu+\lambda)$-elitist   algorithm uses a parent population of size $\mu$. In each iteration,
it generates from it $\lambda$ offspring and determines the next parent population by choosing $\mu$ best individual from the $\mu+\lambda$   parents and offspring. Hence the term ``elitist'' refers to the restriction that the next parent population has to consist of $\mu$ best individuals. Ties can be broken arbitrarily, in particular, in case of ties there is no need to prefer offspring. 

We shall further only regard algorithms that are unbiased in the sense of Lehre and Witt~\cite{LehreW12} (see Section~\ref{subsec:unbiase}).  This in particular means that the algorithm has never access to the bit string representation of individuals (except from using unbiased variation operators and computing the fitness). Consequently, all choices done by the algorithm such as choosing parents for the creation of offspring, choosing variation operators, and selecting the next parent population can only rely on the fitnesses of the individuals created so far. 
Finally, as variation operators we shall only allow unary (mutation)  operators (see Definition \ref{ununopt}). 

In summary, we obtain the $(\mu+\lambda)$-elitist unary unbiased black-box algorithm class described in Algorithm~\ref{1+1}. It is similar to the $(\mu+\lambda)$-elitist model proposed in~\cite{DoerrL17}. Different from~\cite{DoerrL17}, we do not require that the algorithm only has access to a ranking of the search points. We note that for an elitist algorithm, adding this restriction or not does not change a lot. The more significant restriction to~\cite{DoerrL17} is that we require the algorithm to be unary unbiased. We do so since we are trying to first explore simple heuristics, and unary unbiased black-box algorithms are among the most simple search heuristics. 
Nevertheless, to ease the language, we shall in the remainder call our algorithms simply $(\mu+\lambda)$-elitist algorithms, that is, we suppress the explicit mention of the unary unbiasedness. 

\begin{algorithm2e}
Generate $\mu$ search points $x^{(0,i)}$, $i\in[1..\mu]$, independently and uniformly at random\;
$X \assign \{ x^{(0,i)}\mid i\in[1..\mu]\}$\;
\For{$t=1,2,3,\dots$}{
   Choose  $\lambda$ individuals $p_1, \dots ,p_\lambda$ from $ X$\;	
	 Choose  $\lambda$  unary unbiased operators $V_1,\dots,V_\lambda$\;   
   Sample $q_1,\dots,q_\lambda$ from $V_1(p_1),\dots, V_\lambda (p_\lambda)$ respectively\;
   $X \assign $ a selection of $\mu $ best individuals from $X \cup \{q_1,\dots,q_\lambda\}$;
   }
\caption{Template of a $(\mu+\lambda)$-elitist unary unbiased black-box algorithm, $(\mu+\lambda)$-elitist algorithm for short, for optimizing an unknown function $f: \{0,1\}^n \rightarrow \mathbb{R}$. }
\label{1+1}
\end{algorithm2e}

For  a $(\mu+\lambda)$-elitist  algorithm $A$, we recall that the runtime $T_A(f)$ on the maximization problem $f$  is by definition the number of fitness evaluations performed  until a maximum of $f$ is  evaluated for the first time.
The  $(\mu+\lambda)$-elitist black-box complexity of the optimization problem of $f$ is defined to be 
$$
\inf_A E[T_A(f)],
$$
where $A$ runs through all $(\mu+\lambda)$  unary unbiased black-box algorithms (Algorithm \ref{1+1}).

\subsubsection{Independence of Irrelevant Bits}

We start our analysis with a simple lemma showing that bit positions that did not have an influence on any fitness evaluation are independently and uniformly distributed. This lemma is similar to  \cite[Lemma~1]{LehreW12}.

\begin{lemma} \label{lem4}
Let $f: \{0,1\}^n \rightarrow \mathbb{R}$, $c\in \Ima f$, and $I\subset [1..n]$. Let $Z = \{z\in\{0,1\}^n\mid \forall i\notin I: z_i=0  \}$. Assume that for all $y \in \{0,1\}^n$ with $f(y) \le c$, we have $f(y\oplus z)=f(y)$ for all $z \in Z$.

Then for any $(\mu+\lambda)$-elitist   algorithm and any $k\in [1..\mu]$,
conditioning on the event 
$$E_{t,c}\vcentcolon=\{ \max_{j\in[1..\mu]}f(x^{(t,j)})=c\},$$
the bits $x^{(t,k)}_i$, $i\in I$, are all independent and uniformly distributed in $\{0,1\}$.\footnote{We use this language here and in the remainder to express that the bits $x_i$,  $i\in I$, are mutually independent,
independent of all other bits of $x$, 
and uniformly distributed   in $\{0,1\}$.}
\end{lemma}

\begin{proof}
We consider the joint distribution of the random variables $$(x^{(s,j)})_{s\in[0.. t],j\in [1..\mu]}$$ under the condition $E_{t,c}$.
By the hypothesis and the mechanisms of $(\mu+\lambda)$-elitist   algorithms,
$(x^{(s,j)})_{s\in[0.. t],j\in [1..\mu]}$ and $(x^{(s,j)}\oplus z)_{s\in[0.. t],j\in [1..\mu]}$
are identically distributed for any $z \in Z$.

In particular, for any $k\in[1..\mu]$ and any $z\in Z$,
$x^{(t,k)}$ and $x^{(t,k)}\oplus z$ are identically distributed.
From this we deduce that for any $x\in \{0,1\}^n$,
$$
\Pr[x^{(t,k)} = x] = 2^{-|I|} \Pr[x^{(t,k)}_{[1..n] \setminus I} = x_{[1..n] \setminus I}].
$$
Therefore, under the condition $\{x^{(t,k)}_{i} = x_{i},\forall i\in [1..n] \setminus I\}$ where $x_{i}$, $ i\in [1..n] \setminus I$, are prescribed bit values, 
the bits $x^{(t,k)}_i$, $i\in I$, are all independent and uniformly distributed in $\{0,1\}$,
which implies the claim.
\end{proof}

We apply the above lemma  to the optimization  of the $\DLB$ function.
\begin{lemma}\label{indep}
Let $m\in[0..\frac{n}{2}-1]$.
For any $(\mu+\lambda)$-elitist    algorithm and any $k\in [1..\mu]$, 
conditioning on the event 
$$ \left\{\max_{j\in[1..\mu]}\DLB(x^{(t,j)})=2m\right\},$$
or on the event 
$$ \left\{\max_{j\in[1..\mu]}\DLB(x^{(t,j)})=2m+1\right\},$$
the bits $x^{(t,k)}_i$, $i=[2m+3..n]$, are all independent and uniformly distributed in $\{0,1\}$.
\end{lemma}
\begin{proof}
It suffices to take $c=2m$ ($c=2m+1$ for the second case) and $I=[2m+3..n]$ in the preceding lemma.
\end{proof}

\subsubsection{Runtime Analysis of $(1+1)$ Elitist Unary Unbiased EAs}\label{subsec31}
This section is devoted to proving that the     $(1+1)$-elitist black-box complexity of  the   $\DLB$ problem is $\Omega(n^3)$.
To this end we  use drift analysis and take the  $\HLB$ function as potential.

We start with an estimate of the influence of the so-called free-riders,
that is,
we estimate that the expected potential of a random string with $2m+2$ 
leading ones is at most $2m+4$.
\begin{lemma} \label{upperbrand}
Let $m\in[ 0..\frac{n}{2}-1]$.
Let  $x$ be   a random bit string such that $x_i=1$ for all  $i\in[1..2m+2]$
and such that the bits $x_i$,  $i\in[2m+3..n]$ are independent and uniformly distributed in $\{0,1\}$. Then for any $\delta\in[0,2]$
we have
\begin{equation*}
    E[\HLB_\delta(x)]\leq 2m+4.
\end{equation*}
\end{lemma}
\begin{proof}
The statement clearly holds for $m=\frac{n}{2}-1$ since in this case $x=(1,\dots,1)$ and $\HLB_\delta((1,\dots,1))=n< n+2$.
Now we proceed by backwards induction on $m$. 
Suppose that the conclusion  holds for all $m=k+1,\dots,\frac{n}{2}-1$.
For $m=k$ we compute
\begin{align*} 
    E[&\HLB_\delta(x)]\\
    &={}   \tfrac{1}{4}(2k+2)+\tfrac{1}{2}(2k+4-\delta)+\tfrac{1}{4}E[\HLB_\delta(x)\mid x_{2k+3}=x_{2k+4}=1] .
\end{align*}
The induction hypothesis can be applied to the last term, yielding
\begin{align*} 
    E[\HLB_\delta(x)]
       \leq{} & \tfrac{1}{4}(2k+2)+\tfrac{1}{2}(2k+4-\delta)+\tfrac{1}{4} (2k+6) \le 2k+4.
\end{align*}
By induction, this proves the lemma.
\end{proof}
We now estimate the expected progress in one iteration,
first in the case that the parent has an even $\DLB$ value (Lemma \ref{case2l}),
then in the case that it is a local optimum (Lemma \ref{case2l+1}).

\begin{lemma}\label{case2l}
Let $m \in[\frac{n}{4}..\frac{n}{2}-1]$.
Let  $x$ be    a random bit string such that  $\DLB(x)=2m$ 
and that the bits $x_i$, $i\in [2m+3..n]$ are independent and uniformly distributed in $\{0,1\}$.
Let $y$ be a random  bit string  generated from $x$ via a unary unbiased variation  operator $V$.
Let  $$Y\vcentcolon=(\HLB_{\delta}(y)-\HLB_{\delta}(x))\mathbbm{1}_{\DLB(y)\geq\DLB(x)}.$$
Then we have 
\begin{equation*}
E[Y]\leq \frac{2\delta}{n} .
\end{equation*}

\end{lemma}

\begin{proof}
Denote by $r_V$ the random variable describing the number of bits the operator $V$ flips in its argument. Using Lemma \ref{charc_uu}, we can decompose the expectation by conditioning on the value of $r_V$:
\begin{equation}\label{eq13}
E[Y]= \sum_{r=1}^n E[Y\mid r_V=r]\Pr[r_V=r].
\end{equation}
By symmetry  we assume that $x_{2m+1}=1$ and $x_{2m+2}=0$.
We use $F_{2m+1}$ ($F_{2m+2}$ resp.) to denote the event in which
$x_{2m+1}$ ($x_{2m+2}$ resp.) is the only bit that has been flipped among the first $2m+2$ bits of~$x$. 
We observe that $\{Y< 0\}=F_{2m+1}$ and $\{Y> 0\}=F_{2m+2}$.
It follows that
\begin{align*}
E[Y &\mid r_V=r]\\
={}&E[Y\mid r_V=r, F_{2m+1}]\Pr[F_{2m+1}\mid r_V=r]\\
&+E[Y\mid r_V=r, F_{2m+2}]\Pr[F_{2m+2}\mid r_V=r] \\
={}& (\delta-2)\Pr[F_{2m+1}\mid r_V=r]\\
&+E[\HLB_\delta(y)-(2m+2-\delta)         \mid r_V=r, F_{2m+2}]\Pr[F_{2m+2}\mid r_V=r].
\end{align*}
Conditioning on the event $\{r_V=r\}\cap F_{2m+2}$,
the random bit string $y$ has only $1$s in its first $2m+2$ bit positions and the other bits of $y$ are independent and uniformly distributed in $\{0,1\}$. 
 Lemma \ref{upperbrand} thus gives  that  $$E[\HLB_\delta(y)\mid r_V=r, F_{2m+2}]\leq 2m+4.$$
 Therefore we have 
 \begin{align}\label{eq12}
E[Y& \mid r_V=r]  \nonumber \\
&\leq (\delta-2)\Pr[F_{2m+1}\mid r_V=r]+(\delta+2)\Pr[F_{2m+2}\mid r_V=r].
\end{align}
Now we analyze the two probabilities
\begin{align}\label{eq11}
\Pr[F_{2m+1}\mid r_V=r]=&\Pr[F_{2m+2}\mid r_V=r] \nonumber \\
=&{n-2m-2\choose r-1} {n\choose r}^{-1}=\vcentcolon D_r.
\end{align}
We calculate for $r-1 \leq  n - 2m - 2 $ that 
    \begin{align*}
        \frac{D_{r+1}}{D_r}
    ={}&\frac{{n-2m-2\choose r}
    {n\choose r+1}^{-1}}{{n-2m-2\choose r-1}
    {n\choose r}^{-1}}=\frac{(r+1)(n-2m-1-r)}{r(n-r)} \\
    ={}&\frac{n-2m-1-2(m+1)r}{r(n-r)}+1 .
    \end{align*}
Since $r\geq 1$ and $m\geq \frac{n}{4}$, we have
    \begin{equation*}
        n-2m-1-2(m+1)r \leq n-4m-3 <0.
    \end{equation*}
    This implies  that $\frac{D_{r+1}}{D_r} <1$.
Noting also that $D_r = 0$ for $r > n - 2m - 2 $,
we see that $D_r$ is 
decreasing in $r$.
Consequently, we have  
\begin{equation}\label{par1}
    D_r\leq D_1=\frac{1}{n}
\end{equation}
 for all $r \geq 1$.
Combining this with  (\ref{eq12})  and (\ref{eq11}), we obtain
$$
E[Y \mid r_V=r]\leq 2\delta D_r\leq \frac{2\delta}{n}
$$
for all $r\geq 1$.
Substituting this into (\ref{eq13}), the conclusion  follows.
\end{proof}

\begin{lemma}\label{case2l+1}
Let $m\in[\frac{n}{3},\frac{n}{2}-1]$.
Let  $x$ be    a random bit string such that  $\DLB(x)=2m+1$  
and that the bits $x_i$, $i\in [2m+3..n]$, are independent and uniformly distributed in $\{0,1\}$.
Let $y$ be a random bit string  generated from $x$ by a unary unbiased variation  operator $V$.
Let  $Y\vcentcolon=(\HLB_{\delta}(y)-\HLB_{\delta}(x))\mathbbm{1}_{\DLB(y)\geq\DLB(x)}$.
Then we have  
\begin{equation*}
E[Y]\leq \frac{16}{n^2} .
\end{equation*}
\end{lemma}

\begin{proof}
Since $\DLB(x)=2m+1$,
 $\DLB(y)\geq\DLB(x)$ implies  $\HLB_{\delta}(y)\geq \HLB_{\delta}(x)$.
 Thus $Y$ is always non-negative and 
we can  decompose $E[Y]$ into a product of two terms
\begin{align}\label{eq5}
    E[Y]=E[\HLB_{\delta}(y)- \HLB_{\delta}(x)\mid Y>0]\Pr[Y>0].
\end{align}
First we bound the term $E[\HLB_{\delta}(y)- \HLB_{\delta}(x)\mid Y>0]$.
Under the condition $Y>0$, 
we know that $y_i=1$ for $i\in[1..2m+2]$, while the
$y_i$, $i\in [2m+3..n]$, are still independent and uniformly distributed in $\{0,1\}$.
 Lemma \ref{upperbrand} then implies that $E[\HLB_{\delta}(y) \mid Y>0]\leq 2m+4$, thus we have
\begin{equation}\label{par2}
E[\HLB_{\delta}(y)- \HLB_{\delta}(x)\mid Y>0]\leq 2m+4-2m=4.
\end{equation}
On the other hand, to bound $\Pr[Y>0]=\Pr[\DLB(y)>\DLB(x)]$ we invoke Lemma \ref{charc_uu} and compute
\begin{align*}
    \Pr&[\DLB(y)>\DLB(x)]\\
    &= \sum_{r=2}^{n-2m} \Pr[r_V=r]\Pr[\DLB(y)>\DLB(x)\mid r_V=r] \\
    &= \sum_{r=2}^{n-2m} \Pr[r_V=r]{n-2m-2 \choose r-2}{n \choose r}^{-1}\\
    &\leq \max_{2\leq r\leq n-2m} {n-2m-2 \choose r-2}{n \choose r}^{-1},
\end{align*}
where $r$ can be restricted to $2\leq r\leq n-2m$ since otherwise $Y$ would certainly be zero.
With $E_r\vcentcolon ={n-2m-2 \choose r-2}{n \choose r}^{-1}$, we have 
\begin{equation*}\label{quotient1}
    \frac{E_{r+1}}{E_r}=\frac{2(n-m-(m+1)r)}{(r-1)(n-r)}+1 <1
\end{equation*}
by our hypothesis $n-3m\leq 0$ and  $r\geq 2$.
As in Lemma \ref{case2l} we conclude that $E_r$ is decreasing in $r$ and 
\begin{align}\label{par3}
   \Pr[&Y>0]\nonumber\\
   =&\Pr[\DLB(y)>\DLB(x)]\leq  \max_{2\leq r\leq n-2m}E_r=E_2={n \choose 2}^{-1}\leq \frac{4}{n^2}.
\end{align}
Combining (\ref{eq5}), (\ref{par2}) and (\ref{par3}), we obtain $E[Y]\leq \frac{16}{n^2}$.
\end{proof}

With the estimates on the change of the potential \HLB above, we can now easily prove the main result of this section. By taking $\delta = \Theta(\frac 1n)$ suitably, we ensure that the expected gain of the potential in one iteration is only $O(\frac{1}{n^2})$, at least in a sufficiently large range of the potential space.
To have a small progress in the whole search space, as necessary to apply the additive drift theorem, we regard the potential  $h_t=\HLB_\delta(x^{(t+S)})$, where $S$ denotes the first time that $x^{(t)}$ has an $\HLB$ value not less than $\frac{2n}{3}$. This potential also has a drift of only $O(\frac{1}{n^2})$ and thus the additive drift theorem easily gives the lower bound of $\Omega(n^3)$.

\begin{theorem}\label{1+lam}
The $(1+1)$-elitist black-box complexity of the $\DLB$ problem is $\Omega(n^3)$.
\end{theorem}
\begin{proof}
Consider any $(1+1)$-elitist algorithm $A$. Let $x^{(t)}$ denote the individual at time $t$.
By a slight abuse of notation we write  $\frac{2n}{3}$ to denote the even integer~ $ \min \{\ell\in 2\mathbb{N}\mid \ell\geq \frac{2n}{3}\}$.
First we observe that $\Pr\left[\HLB(x^{(0)})\geq \frac{2n}{3}\right]=\Theta(2^{-\frac{2}{3}n})$. 
In the following we work under the condition $\mathcal{C}\vcentcolon=\{\HLB(x^{(0)})<\frac{2n}{3}\} $. 
Under this condition,  $S\vcentcolon = \min\{t\in\mathbb{N}\mid \HLB(x^{(t)})\geq \frac{2n}{3}\}$ is strictly positive.
Let $T\vcentcolon = \min\{t\in\mathbb{N}\mid \HLB(x^{(t)})\geq n\}+1$ be the runtime of $A$.
We will use the additive drift theorem (Theorem \ref{adddrift}) 
to prove $E[T-S\mid \mathcal{C}]= \Omega(n^3)$,
from which   the claim follows via
\begin{equation}\label{eq15}E[T]\geq \Pr[\mathcal{C}]E[T\mid \mathcal{C}]\geq (1-o(1)) E\left[T-S\mid\mathcal{C}\right]. 
\end{equation}

For a $\delta$ to be specified later, we regard the potential $h_t=\HLB_\delta(x^{(t+S)})$. Note that because of the elitism of our algorithm, we have $h_t \ge \frac{2n}{3}$ for all $t \ge 0$.
Let  $m\in[\frac{n}{3}..\frac{n}{2}-1]$.

Lemma \ref{indep}  implies that when conditioning on the event $\{h_t=2m+2-\delta\}=\{\DLB(x^{(t+S)})=2m\}$, the bits
$x^{(t+S)}_i$, $i\in[2m+3..n]$, are independent and uniformly distributed in $\{0,1\}$.
Let $y$ be a bit string  obtained by applying a unary unbiased variation operator to  $x^{(t+S)}$.
 Lemma \ref{case2l} then applies to $y$ and $x^{(t+S)}$, yielding
 $$
    E\left [(\HLB_{\delta}(y)-\HLB_{\delta}(x^{(t+S)}))\mathbbm{1}_{\DLB(y)\geq\DLB(x^{(t+S)})}\,\middle| \,\mathcal{C}\right ]\leq \frac{2\delta}{n}.
    $$
Since the term on the left-hand side is the   expected increase in the $\HLB$  function during an iteration,
we have in fact proven that $\Delta_t(2m+2-\delta)\vcentcolon=E[h_{t+1}-h_t\mid h_t=2m+2-\delta,\mathcal{C}]\leq \frac{2\delta}{n}$ for any $t$.

For the case $\{h_t=2m\}$ we  proceed in the same manner (except that we use Lemma \ref{case2l+1} instead of  Lemma  \ref{case2l})  to obtain     $\Delta_t(2m)\leq \frac{16}{n^2}$ for any~ $t$.
By setting  $\delta$ to $\frac{8}{n}$,
  $\Delta_t(s)\leq \frac{16}{n^2}$ holds for any possible $h_t$ value~ $s$.

Recall that $\frac{2n}{3}$ denotes the smallest even integer not less than it.
By the definition of $S$ and by Lemma \ref{indep},
 $x^{(S)}$ is a random bit string satisfying $x^{(S)}_i=1$ for $i\in[1..\frac{2n}{3}]$ and such that $x^{(S)}_i$, $i\in[\frac{2n}{3}+1..n]$, are independent and uniformly distributed in $\{0,1\}$.
Thus  we   have $E[h_0]=E[\HLB_\delta(x^{(S)})]=\frac{2n}{3}+O(1)$ by virtue of Lemma~ \ref{upperbrand}.
The additive drift theorem (Theorem \ref{adddrift})  then yields  $$E[T-S\mid\mathcal{C}]\geq \frac{n-\frac{2n}{3}-O(1)}{16/n^2}=\Theta(n^3),$$
which finishes the proof with (\ref{eq15}).
 \end{proof}

\subsection{The Unary Unbiased Black-box Complexity of the DLB Problem is at Most Quadratic}\label{sec:nonel}

In the preceding  section we have seen that  elitism does not ease optimizing the $\DLB$ problem.
 This section, therefore, is devoted to investigating the best possible expected runtime on the $\DLB$ problem without elitism.
To be more precise, we  show that the unary unbiased black-box complexity of the $\DLB$ problem is $O(n^2)$. This result will be complemented by a matching lower bound (Theorem \ref{Una_Unbia}) in Section \ref{sec2}.

%
%
%
%
%
%

We recall the definition of a $k$-ary unbiased black-box algorithm (Algorithm~\ref{alg:unbiased}) from Section~\ref{subsec:unbiase}.
The $k$-ary unbiased black-box complexity of a problem is defined as 
the infimum expected runtime of a $k$-ary unbiased black-box algorithm on this problem, that is,
$$
\inf_{A}E[T_A(f)],
$$
where $A$ runs through all $k$-ary unbiased black-box algorithms.

Let $V$ be the unary unbiased operator such that $V(x)$ is obtained from flipping one bit of $x$. Then the additive drift theorem implies  that the $(1+1)$-elitist  black-box algorithm using $V$ as variation operator has an   expected runtime of  $O(n^2)$ on the $\lo$ problem. 
By a small adjustment of this $(1+1)$-elitist black-box algorithm, we  exhibit a simple unary unbiased black-box algorithm that solves the $\DLB$ problem in expected time $\Theta(n^2)$. This inspires our investigation on the expected runtime of the Metropolis algorithm in Section \ref{sec3}. 
\begin{lemma}\label{Una_Unbia_upper_bound}
The unary unbiased black-box complexity of the $\DLB$ problem is $O(n^2)$.
\end{lemma}

\begin{proof}
We present a simple algorithm and then show that its expected runtime is indeed $O(n^2)$.
Throughout the algorithm we only use  the  unary unbiased operator $V$ that flips exactly one randomly chosen bit. 

The algorithm is initialized by generating a search point $x^{(0)}$ at random in $\{0,1\}^n$.
In generation $t$, we generate a bit string $y=V(x^{(t)})$.
If $\DLB(x^{(t)})$ is even and $\DLB(x^{(t)})<\DLB(y)$ we accept $y$ as new search point, i.e., $x^{(t+1)}\vcentcolon=y$.
If  $\DLB(x^{(t)})$ is odd and $\DLB(x^{(t)})<2+\DLB(y)$,
we also accept $y$ as new search point, i.e., $x^{(t+1)}\vcentcolon=y$.
In all other cases, we reject $y$, i.e., $x^{(t+1)}\vcentcolon=x^{(t)}$.

Now we show that this algorithm finds the optimum of the $\DLB$ problem in time $O(n^2)$.
We define the potential at time $t$ as  $k_t=\HLB_{\frac{3}{2}}(x^{(t)})$.
In the case where $k_t=2m$, that is, $\DLB(x^{(t)})$ is odd, we have a $00$ critical block and the probability that we flip one of its two zeros is $\frac{2}{n}$,
yielding an expected gain of at least $\frac{2}{n}\cdot\frac{1}{2}=\frac{1}{n}$ in the potential.
In the case where $k_t=2m+\frac{1}{2}$, that is, $\DLB(x^{(t)})$ is even, the critical block contains exactly one zero and one one.
If the bit with value $0$ is flipped, we increase the potential from $2m+\frac 12$ to at least $2(m+1)$. If the bit with value $1$ is flipped, we reduce the potential from $2m+\frac 12$ to $2m$. All other bit-flips are not accepted or do not change the potential. Consequently, the expected gain in potential is at least
$\frac 1n (2(m+1)-(2m+\frac12)) + \frac 1n (2m-(2m+\frac12)) = \frac 1n$.
Since the potential needs to be increased by at most $n$, 
the additive drift theorem (Theorem \ref{adddrift}) establishes  $\frac{n}{1/n}=n^2$
as an upper bound for the expected  runtime of this algorithm on the $\DLB$ function.
\end{proof}

We remark that   the artificial algorithm defined in the proof is a $(1+1)$-type unbiased  algorithm, therefore we have the following stronger result.
 \begin{theorem}\label{thm:uubbcUB}
The $(1+1)$-type unbiased black-box complexity of the $\DLB$ problem is $O(n^2)$.
\end{theorem}

\subsection{The Metropolis Algorithm Performs Well on the DLB Problem} \label{sec3}

Inspired by the analysis in the preceding sections, 
we expect that certain non-elitist $(1+1)$-type unbiased search heuristics outperform elitist EAs on the $\DLB$ problem. 
In fact, we will prove now that the Metropolis algorithm 
(simulated annealing with a fixed temperature) can solve the DLB problem within $\Theta(n^2)$ fitness evaluations in expectation.
This performance coincides with the unary unbiased black-box complexity of the DLB function (Theorem \ref{Una_Unbia}).

The Metropolis algorithm is a simple single-trajectory optimization heuristic. In contrast to elitist algorithms like randomized local search or the \oea, it can accept inferior solutions, however only with a small probability that depends on the degree of inferiority and an algorithm parameter $\alpha \in (1,\infty)$. 

More precisely, the maximization version of the Metropolis algorithm works as follows. It starts with a random initial solution $x^{(0)}$. In each iteration $t = 1, 2, \dots$, it generates a random neighbor $y$  of the current search point~$x^{(t-1)}$. When working with a bit string representation (as in this work), such a neighbor is obtained from flipping in $x^{(t-1)}$ a single bit chosen uniformly at random. If $f(y) \ge f(x^{(t-1)})$,
then the algorithm surely accepts $y$ as new search point $x^{(t)}\vcentcolon=y$.
If $y$ is inferior to $x^{(t-1)}$, it accepts $y$ ($x^{(t)}\vcentcolon=y$) only with probability $\alpha^{f(y)-f(x^{(t-1)})}$ and otherwise rejects it ($x^{(t)}\vcentcolon=x^{(t-1)}$).
We note that the probability $\alpha^{f(y)-f(x^{(t-1)})}$ for accepting an inferior solution is often written as $\exp((f(y)-f(x^{(t-1)})) / kT)$ for a ``temperature parameter'' $kT$, but clearly the two formulations are equivalent.  
The pseudocode for this algorithm is given in  Algorithm \ref{metrop}.

\begin{algorithm2e}
Generate a search point $x^{(0)}$  uniformly in $\{0,1\}^n $;

\For{$t=1,2,3,\dots$}{
   Choose  $i\in[1.. n]$ uniformly at random and obtain $y$ from  flipping the $i$-th bit in  $x^{(t-1)}$;
   
   \eIf{$f(y) \ge f(x^{(t-1)})$}{
  $x^{(t)}\leftarrow y$\;
   }{
   Choose $b \in \{0,1\}$ randomly with $\Pr[b=1]=\alpha^{f(y)-f(x^{(t-1)})}$\;
   \eIf{$b=1$}{
    $x^{(t)}\leftarrow y$\;
   }{
   $x^{(t)}\leftarrow x^{(t-1)}$\;
   }
  }
}
\caption{Metropolis algorithm for maximizing a function $f: \{0,1\}^n \rightarrow \mathbb{R} $}
\label{metrop}
\end{algorithm2e}

Now we show that the Metropolis algorithm with sufficiently large parameter $\alpha$ can solve the $\DLB$ problem in  time quadratic in $n$.
To this end, we show the following elementary lemma.
\begin{lemma}\label{const}
For all $\alpha>\sqrt{2}+1$, $$C(\alpha)\vcentcolon=\frac{2}{\alpha}\left(\frac{1}{2}-2\sum_{k=1}^\infty k \alpha^{-2k}\right)$$
is a strictly positive constant.
\end{lemma}

\begin{proof}

We observe that
$$\alpha^{-2}\sum_{k=1}^\infty k \alpha^{-2k}=\sum_{k=1}^\infty k \alpha^{-2(k+1)}=\sum_{k=2}^\infty (k-1) \alpha^{-2k}.$$
Subtracting $\sum_{k=1}^\infty k \alpha^{-2k}$ from both sides of the equation yields
$$
(\alpha^{-2}-1)\sum_{k=1}^\infty k \alpha^{-2k}=\sum_{k=2}^\infty -\alpha^{-2k}-\alpha^{-2}=-\sum_{k=1}^{\infty} \alpha^{-2k}=-\frac{\alpha^{-2}}{1-\alpha^{-2}},
$$
which implies 
$$
\sum_{k=1}^\infty k \alpha^{-2k}=\frac{\alpha^{-2}}{(1-\alpha^{-2})^2}=\frac{\alpha^{2}}{(\alpha^{2}-1)^2}.
$$
Hence we have  
$$C(\alpha)=\frac{2}{\alpha}\left(\frac{1}{2}-2\frac{\alpha^{2}}{(\alpha^{2}-1)^2} \right)=\frac{\alpha^4-6\alpha^2+1}{\alpha(\alpha^2-1)^2}.$$

For all $\alpha>\sqrt{2}+1$, we have $\alpha^4-6\alpha^2+1>0$, which    implies $C(\alpha)>0$.
\end{proof}

We now show the main result of this section that the Metropolis algorithm can optimize the DLB function in quadratic time if the selection pressure is sufficiently high, that is, $\alpha$ is a large enough constant.

\begin{theorem}\label{MA_runtime}
The expected runtime of the Metropolis algorithm on the DLB problem is at most $\frac{n^2}{C(\alpha)}$, provided that the parameter $\alpha$ satisfies $\alpha > \sqrt{2}+1$.
\end{theorem}

To prove this result, we need to argue that the negative effect of accepting solutions further away from the optimum is outweighed by the positive effect that a critical $00$-block can be changed into a critical block $01$ or $10$ despite the fact that this decreases the DLB value. To achieve this, we design a suitable potential function, namely the HLB function introduced in Section~\ref{sec:dlb} with parameter $\delta = \tfrac 32$ and show that each iteration (starting with a non-optimal search point) in expectation increases this potential by $\Omega(\frac 1n)$. With this insight, the additive drift theorem immediately gives the claim. 

\begin{proof}[Proof of Theorem~\ref{MA_runtime}]
We denote by $x^{(t)}$ the search point obtained at the end of iteration $t$.
To apply drift analysis we take $\HLB_{\frac{3}{2}} $  as potential and abbreviate $\HLB_{\frac{3}{2}}(x^{(t)})$  by $h_t$. 
Recalling that the Metropolis algorithm generates each Hamming neighbor $y$ of $x^{(t-1)}$ with equal probability $\frac 1n$, but accepts this only with probability $\alpha^{\min\{0,f(y)-f(x^{(t-1)})\}}$, we compute for each $m\in [0.. \frac{n}{2}-1]$ that
\begin{align*}
 E\Big[h_{t}&-h_{t-1} \,\Big|\, h_{t-1}=2m+\frac{1}{2}\Big] \\
        \geq{} & \frac{1}{n}\left(2m+2-\left(2m+\frac{1}{2}\right)\right)
        + \frac{1}{n} \left(2m-\left(2m+\frac{1}{2}\right)\right)  \\
        &+\sum_{k=0}^{m-1}\frac{2}{n} \alpha^{2k-2m}
        \left(2k+\frac{1}{2}-\left(2m+\frac{1}{2}\right)\right)\\
        ={}&\frac{1}{n}  +\frac{2}{n}\sum_{k=1}^{m} \alpha^{-2k}
        \left(-2k\right) =\frac{2}{n}\left(\frac{1}{2}-2\sum_{k=1}^m k \alpha^{-2k}\right)
\end{align*}
and 
\begin{align*}
E[h_{t}&-h_{t-1}\mid h_{t-1}=2m] \\
\geq{} & \frac{2}{n}\frac{1}{\alpha}\left(2m+\frac{1}{2}-2m\right) +\sum_{k=0}^{m-1}\frac{2}{n} \alpha^{2k-2m-1} \left(2k+\frac{1}{2}-2m\right) \\
={} &\frac{2}{n\alpha}\frac{1}{2}+\sum_{k=1}^{m}\frac{2}{n\alpha} \alpha^{-2k} \left(\frac{1}{2}-2k\right)
> \frac{2}{n\alpha}\left(\frac{1}{2}-2\sum_{k=1}^m k \alpha^{-2k}\right).
\end{align*}
We have shown in Lemma \ref{const} that $\frac{1}{2}-2\sum_{k=1}^\infty k \alpha^{-2k}>0$ for all $\alpha>\sqrt{2}+1$.
Thus for any $s\neq n$,
\begin{align*}
E[h_{t}&-h_{t-1}\mid h_{t-1}=s] \\
\geq{} &\frac{2}{n\alpha}\left(\frac{1}{2}-2\sum_{k=1}^m k \alpha^{-2k}\right) 
\geq \frac{2}{n\alpha}\left(\frac{1}{2}-2\sum_{k=1}^\infty k \alpha^{-2k}\right) = \frac{C(\alpha)}{n}.
\end{align*}
The additive drift theorem (Theorem \ref{adddrift}) now implies that the expected runtime of Metropolis 
algorithm on the $\DLB$ problem is bounded by 
$$  \frac{n}{C(\alpha)/n}=\frac{n^2}{C(\alpha)},$$
which concludes the proof.
\end{proof}

\subsection{Literature Review on the Metropolis Algorithm and Non-Elitist Evolutionary Algorithms}

To put our results on the Metropolis algorithm into context, we now brief{}ly survey the known runtime results on this algorithm and non-elitist evolutionary algorithms, which are surprisingly few compared to the state of the art for elitist search heuristics. 

\subsubsection{Metropolis Algorithm}\label{sssec:metro} 
We first note that the Metropolis algorithm is a special case of simulated annealing, which can be described as the Metropolis algorithm using a temperature that is decreasing over time, that is, the parameter $\alpha$ increases over time in the language of Algorithm~\ref{metrop}. We refer to~\cite{DowslandT12} for a survey of the main applied results and restrict ourselves to the known theory results. Here the survey of Jansen~\cite{Jansen11}, even though from 2011, still is a good source of information.  

Already in 1988, Sasaki and Hajek~\cite{SasakiH88} showed that the Metropolis algorithm (and also its generalization simulated annealing) has an at least exponential worst-case runtime on the maximum matching problem. On the positive side, the Metropolis algorithm with constant temperature (constant $\alpha$) can compute good approximate solutions. Analogous results, however, have also been shown for the \oea by Giel and Wegener~\cite{GielW03}. Jerrum and Sorkin~\cite{JerrumS98} showed that the Metropolis algorithm can solve the minimum bisection problem in quadratic time in random instances in the planted bisection model. Wegener~\cite{Wegener05} provided a simple instance of the minimum spanning tree (MST) problem, which can be solved very efficiently by simulated annealing with a natural cooling schedule, but for which the Metropolis algorithm with any temperature needs an exponential time to find the optimum. A similar result for the traveling salesman problem was given by Meer~\cite{Meer07}. In fact, Wegener~\cite{Wegener05} proved that simulated annealing in polynomial time finds the optimum of any MST instance with $\eps$-separated weights. His conjecture that simulated annealing in polynomial time computes $\eps$-approximations to any instance was recently proven in~\cite{DoerrRW22}.

Jansen and Wegener~\cite{JansenW07} proved that the Metropolis algorithm with $\alpha \ge \eps n$, that is, with a very small temperature in the classic language, optimizes the $\onemax$ benchmark in time $O(n \log n)$, a runtime also known for many simple evolutionary algorithms~\cite{Muhlenbein92,JansenJW05,Witt06,AntipovD21algo}. They further show that this runtime is polynomial if and only if $\alpha = \Omega(n / \log n)$. Consequently, in the classic language, only for very small temperatures (that is, with very low probabilities of accepting an inferior solution) the Metropolis algorithm is efficient on the \onemax benchmark. An extension of this characterization to arbitrary symmetric functions (that is, functions $f$ that can be written as $f(x) = g(\onemax(x))$ for some $g : \R \to \R$) was given in~\cite{KadenWW09}. While it is interesting that such a characterization could be obtained, the characerization itself remains hard to interpret. Jansen and Wegener~\cite{JansenW07} further designed several example problems showing different or similar runtimes of the Metropolis algorithm and the \oea, among them the function~$f_2$, which can be solved by the Metropolis algorithm with $\alpha = n$, that is, again a very small temperature, in polynomial expected time, whereas the \oea needs time $n^{\Omega(\log\log n)}$, and another function for which the Metropolis algorithm with any parameter setting has an expected runtime of $\Omega(2^{0.275n})$, whereas the \oea has an expected runtime of only $\Theta(n^2)$.

In their work on the Metropolis algorithm and the Strong Selection Weak Mutation (SSWM) algorithm, Oliveto, Paix\~ao, Heredia, Sudholt, and Trubenov\'a~\cite{OlivetoPHST18} proposed the \valley problem, which contains a fitness valley with descending slope of length $\ell_1$ and depth $d_1$ and ascending slope of length $\ell_2$ and height $d_2$. This valley is constructed onto a long path function, making this essentially a one-dimensional optimization problem (despite being defined on $\{0,1\}^n$). They proved rigorously that the Metropolis algorithm takes an expected number of $n \alpha^{\Theta(d_1)}+\Theta(n\ell_2)$ function evaluations to cross this valley of low fitness when $\alpha$ (in the notation of Algorithm~\ref{metrop}, note that the $\alpha$ used in~\cite{OlivetoPHST18} has a different meaning) is at least $\alpha \ge \exp(c \max\{\ell_1/d_1, \ell_2 / d_2\})$ for a sufficiently large constant $c$. A similar result holds for the SSWM algorithm. Since the \oea needs time $\Omega(n^{\ell_1})$ to cross the valley, here the Metropolis and SSWM algorithm are clearly superior.
%
%
%

In their time complexity analysis of the move acceptance hyper-heuristic (MAHH), Lissovoi, Oliveto, and Warwicker~\cite{LissovoiOW23} also consider the Metropolis algorithm. For the multimodal \cliff benchmark with constant cliff length~$d$, they show that the Metropolis algorithm needs at least an expected number of $\Omega((\frac{n}{d})(\frac{n}{\log n})^{d-1.5})$ iterations to find the optimum, which is at most a little faster than the $\Theta(n^d)$ runtime of simple mutation-based algorithms. However, this is much worse than the $O(n \log n)$ performance of a self-adjusting $(1,\lambda)$~EA~\cite{FajardoS21gecco} (only proven for $d=n/3$), the $O(n^3)$ runtime of the MAHH~\cite{LissovoiOW23}, the $O((n/d)^2 n\log n)$ runtime of two self-adjusting fast artificial immune systems~\cite{CorusOY21foga}, and the $O(n^{3.9767})$ runtime of the $(1,\lambda)$~EA with the best static choice of $\lambda$~\cite{FajardoS21gecco} (improving the well-known $O(n^{25})$ runtime guarantee~\cite{JagerskupperS07}), again only for $d=n/3$. For the multimodal \jump benchmark, the Metropolis algorithm was shown~\cite{LissovoiOW23} to have a runtime exponential in $n$ with high probability regardless of the jump size $m$ when the temperature is sufficiently small that the local optimum of the jump function is reached efficiently. This compares unfavorably with the known runtimes of many algorithms, which are all $\Theta(n^m)$ or better, see, e.g.,~\cite{DrosteJW02,DangFKKLOSS16,DangFKKLOSS18,DoerrLMN17,HasenohrlS18,AntipovBD21gecco,Doerr21cgajump,Doerr22,RajabiW22,RajabiW23,DoerrR23,DoerrEJK23arxiv}.

In summary, we would not say that these results give a strong recommendation for using the Metropolis algorithm to optimize pseudo-Boolean functions. It excels on the $f_2$ function proposed in~\cite{JansenW07} and the \valley problem proposed in~\cite{OlivetoPHST18}, however, problems essentially are one-dimensional problems and thus of a very different structure than most pseudo-Boolean problems. The possible $\tilde O(n^{0.5})$ runtime advantage over simple EAs on the \cliff problem is interesting (if it exists, only a lower bound was shown in~~\cite{LissovoiOW23}), but appears still small when recalling that the runtime of the simple EAs is $\Theta(n^d)$. Hence maybe most convincing result favoring the Metropolis algorithm is the quadratic runtime shown~\cite{JerrumS98} for the random bipartition instances. Here, however, one also has to recall that random instances of $NP$-complete problems can by surprisingly easy. For example, the \oea can solve random 3-SAT instances in the planted solution model in time $O(n \log n)$ when the clause density is at least logarithmic~\cite{DoerrNS17}. These mildly positive results on the Metropolis algorithm have to be contrasted with several negative results, e.g., the exponential runtimes shown in~\cite{SasakiH88, Wegener05, Meer07}. Also, the parameter choice seems to be non-trivial. Whereas the best results in~\cite{OlivetoPHST18} are obtained by small values of $\alpha$, recall that the runtime shown there is at least $n \alpha^{\Omega(d_1)}$ and the simple \onemax benchmark can only be optimized in polynomial time when $\alpha$ is as large as $\Omega(n \log n)$. 

For reasons of completeness, we note that after the completion of this work a runtime analysis of the Metropolis algorithm on a generalized \cliff benchmark has appeared~\cite{DoerrERW23}. It showed that the Metropolis algorithm for most parameters of the benchmark has a performance inferior to the one of simple elitist algorithms.

\subsubsection{Non-Elitist Evolutionary Algorithms}\label{sssec:nonel}
The strategy to not stick with the best-so-far solution to enable the algorithm to leave a local optimum is also well-known in evolutionary computation. However, also there the rigorous support for such non-elitist algorithms is rather weak. 

There is a fairly elaborate methodology to  prove upper bounds on the runtime of non-elitist algorithms~\cite{Lehre11,DangL16algo,CorusDEL18,DoerrK21algo}, however, these tools so far could mostly be applied to settings where the evolutionary algorithm loses the best-so-far solution so rarely that it roughly imitates an elitist algorithm. A few analyses of particular algorithms point into a similar direction~\cite{JagerskupperS07,RoweS14}. The existing general lower bound methods for non-elitist algorithms~\cite{Lehre10,Doerr21ecjLB} in many cases allowed to prove that with just a little more non-elitism, the algorithm has no chance to optimize efficiently any pseudo-Boolean function with unique optimum. 

As a concrete example, the standard \mclea using mutation rate $\frac 1n$ cannot optimize any pseudo-Boolean function with unique optimum in sub-exponential time if $\lambda \le (1-\eps) e \mu$, $\eps > 0$ any constant~\cite{Lehre10,Doerr21ecjLB}. However, when $\lambda \ge (1+\eps) e \mu$ and $\lambda = \omega(\log n)$, then the \mclea optimizes \onemax and \jump functions with constant jump size in essentially the same runtime as the \mplea~\cite{Lehre11,DoerrK21algo}. This suggests that there is at most a very small regime in which non-elitism can be successfully exploited. Such a small middle regime was shown to exist when optimizing \cliff functions via the \oclea. More specifically, J\"agersk\"upper and Storch~\cite{JagerskupperS07} showed that the \oclea with $\lambda \ge 5 \ln n$ optimizes the \cliff function with cliff length $n/3$ in time $\exp(5 \lambda) \ge n^{25}$, whereas elitist mutation-based algorithms easily get stuck in the local optimum and then need at least $n^{n/3}$ expected time to optimize this \cliff function. This result was very recently improved and runtime of $O(n^{3.9767})$ was shown for the asymptotically best choice of $\lambda$~\cite{FajardoS21gecco}. For the \mclea optimizing jump functions, however, the existence of a profitable middle regime was disproven in~\cite{Doerr22}. Some more results exist that show situations where only an inefficient regime exists, e.g., when using (1+1)-type hillclimbers with fitness-proportionate selection to optimize linear pseudo-Boolean functions~\cite{HappJKN08} or when using a mutation-only variant of the simple genetic algorithm to optimize \onemax~\cite{NeumannOW09}. For the true simple genetic algorithm, such a result exists only when the population size is at most $\mu \le n^{1/4 - \eps}$, but there is no proof that the situation improves with larger population size~\cite{OlivetoW15}. 

The few results showing that non-elitism can help to leave local optima include, besides the examples for the Metropolis algorithm discussed in Section~\ref{sssec:metro} (see the last paragraph of that section for a summary) and the two ``small middle regime'' results discussed in the previous paragraph, the following. Dang, Eremeev, and Lehre~\cite{DangEL21aaai} showed that a non-elitist EA with 3-tournament selection, population size $\lambda \ge c \log n$ for $c$ a positive constant, and bitwise mutation with rate $\chi/n$,  $\chi=1.09812$, can reach the optimum of the multi-modal \funnel problem with parameter $\omega \le \tfrac34 n$ in expected time $O(n\lambda \log \lambda + n^2 \log n)$, whereas the \mplea and the \mclea cannot reach the optimum in time $2^{c'n}$, $c'>0$ a suitable constant, with overwhelming probability. In~\cite{DangEL21gecco}, the same authors defined a class of functions having an exponential elitist $(\mu+\lambda)$ black-box complexity which can be solved in polynomial time by several non-elitist algorithms. 
Fajardo and Sudholt~\cite{FajardoS21gecco} showed that a self-adjusting variant of the \oclea can optimize \cliff functions in time $O(n \log n)$. Zheng, Zhang, Chen, and Yao~\cite{ZhengZCY21} proved that the \oclea with offspring population size $\lambda=c\log_{\frac{e}{e-1}} n$ for the constant $c\ge 1$ can reach the global optimum of the time-linkage \OM function in expected time $O(n^{3+c}\log n)$, while the \oplea with $\lambda \in [e^e, e^{n^{1/3}}]$ (as well as the \oea analyzed in~\cite{ZhengCY21}) with $1-o(1)$ probability cannot reach the global optimum in arbitrary long time. Very recently, Jorritsma,  Lengler, and Sudholt~\cite{JorritsmaLS23} showed that comma selection offers moderate advantages on randomly perturbed \onemax functions.

More results exist for non-classical algorithms. In~\cite{PaixaoHST17}, the strong-selection weak-mutation process from biology is regarded under an algorithmic perspective. It is shown that when setting the two parameters of this algorithm right, then it optimizes the $\cliff_d$ function, $d = \omega(\log n)$, in time $(n/\Omega(1))^d$, which is faster than the \oea by a factor of $\exp(\Omega(d))$. It also optimizes the slightly artificial \balance function with high probability in time $O(n^{2.5})$, whereas the \oea takes weakly exponential time. In~\cite{LissovoiOW23}, a runtime analysis of the move-acceptance hyper-heuristic (MAHH) is performed. This can be seen as a variant of the Metropolis algorithm, where the probability of accepting an inferior solution is an algorithm parameter independent of the degree of inferiority. If this probability is chosen sufficiently small, then the MAHH optimizes all \cliff functions in time $O(n^3)$, significantly beating all other results on this function class. For the $\jump_k$ function, however, the MAHH is not very effective. For all $k = o(\sqrt n)$ and all values of the acceptance parameter~$p$, a lower bound of $\Omega(\frac{n^{2k-1}}{(2k-1)!})$ was proven in~\cite{DoerrDLS23}.  We note that aging as used in artificial immune systems can also lead to non-elitism and has been shown to speed-up leaving local optima, see, e.g.,~\cite{OlivetoS14,CorusOY20,CorusOY21foga,CorusOY21tec}. Since the concept of aging has not found many applications in heuristic search outside the small area of artificial immune systems used as optimizers, we omit a detailed discussion. Finally, we note that restarts obviously lead to algorithms that can be called non-elitist. We also discuss these not further since we view them, similar as parallel run or racing techniques, rather as mechanisms to be added to basic optimization heuristics.
%
Overall, these examples prove that non-elitism can be helpful, but from their sparsity in the large amount of research on randomized search heuristics, their specificity, and their sometimes still high runtimes, we would not interpret them as strong indication for trying non-elitist approaches early.

\subsection{A Lower Bound for the Unary Unbiased Black-Box Complexity} \label{sec2}

In this section we will  show that the expected runtime of any unary unbiased black-box algorithm  on the $\DLB$ problem  is $\Omega(n^2)$, which together with the upper bound of Section 
\ref{sec:nonel} proves that the unary unbiased black-box complexity of the $\DLB$ problem is $\Theta(n^2)$. This result is very natural given that Lehre and Witt~\cite{LehreW12} have proven the same lower bound for the unary unbiased black-box complexity of the $\lo$ problem, which appears rather easier than the $\DLB$ problem.

To this end, we first prove the following lemma, which affirms that for a search point in a certain range of $\lo$ values, the probability of obtaining a search point with better fitness after one step is $O(\frac{1}{n})$.

\begin{lemma}\label{calcul}
Suppose that   $n $ is an even integer and that $m$ is an integer such that $\frac{n}{4}< m<\frac{n}{3}$.
Let $x$ be a bit string
and let $\ell\vcentcolon=|\{j\in[1..2m+2]\mid x_j=1\}|$ designate  the number of ones in the  leading $2m+2$ bit positions of $x$.  
We suppose that $1\leq \ell\leq 2m+1$.
Then for any   unary unbiased operator~$V$, we have
 $\Pr[\lo(V(x))\ge 2m+2] \le \frac{2}{n}$.

\end{lemma}

 \begin{proof}
We   assume  that for some $r \in[0..n]$, $V $  is the unary unbiased operator that  flips $r$ bit positions chosen uniformly at random from the bit string. 
The conclusion for the  general case then follows by applying the law of total probability and Lemma \ref{charc_uu}.
 
We furthermore assume that  $2+2m-\ell\leq r\leq n-\ell$ since otherwise the claim is trivial.
Indeed, $r<2+2m-\ell$ implies  $r+\ell<2m+2$. Since $r+\ell$ is an upper bound on the number of ones $V(x)$ has in the first $2m+2$ positions, we obtain $\lo(V(x))<2m+2$ with probability $1$.
If $r>n-\ell$, then at least one of the bits
$x_j=1$, $j\in[1..2m+2]$, is flipped,
also resulting in $\lo(V(x))<2m+2$.

 The event $\{\lo(V(x))\ge 2m+2\}$ happens if and only if  all bits $x_j$ such that   $x_j=1$ and $j\in[1..2m+2]$, are not flipped,
  and  all bits $x_j=0$, $j\in[1..2m+2]$, are  flipped.
 Among all ${n \choose r}$ ways of flipping $r$ bit positions in $x$, 
 exactly ${n-2m-2 \choose r-(2m+2-\ell)}$ of them flip all the $0$s in the leading $2m+2$ bit positions and leave all the $1$s in the leading $2m+2$ bit positions unchanged.
Hence the probability $D_\ell(r)$ of the event $\{\lo(V(x))\ge 2m+2\}$  is given  by
 $$D_\ell(r)=\frac{{n-2m-2 \choose r-(2m+2-\ell)}}{{n \choose r}}=\frac{(n-r)  \cdots(n-r-\ell+1)\cdot r \cdots(r-2m+\ell-1)}{n \cdots   (n-2m-1)}.$$
From our assumptions, with the variables $n$, $m$ and  $r$ fixed, we have
 \begin{equation}\label{range}
    \ell \in [\max \{1,2m+2-r\}.. \min \{2m+1,n-r\}].
 \end{equation}
For $\ell\in[\max \{1,2m+2-r\}.. \min \{2m,n-r-1\}]$, we compute
 \begin{equation}\label{compar}
 \frac{D_{\ell+1}(r)}{D_\ell(r)}=\frac{2(\frac{n}{2}+m-\ell-r)+1}{r-2m+\ell-1}+1.
 \end{equation}
We conclude that   $D_{\ell+1}(r)>D_{\ell}(r)$ ($D_{\ell+1}(r)<D_{\ell}(r)$ resp.)  is equivalent to  $\frac{n}{2}+m\ge \ell+r$ ($\frac{n}{2}+m< \ell+r$ resp.).
 To proceed we distinguish three cases regarding the value of $r$. 
 
 \textbf{Case} $r\in[1.. \frac{n}{2}-m-1]$.
 
  Since $\ell\leq 2m+1$ (equation (\ref{range})), we have $\ell+r  \leq  2m+1+\frac{n}{2}-m-1 =\frac{n}{2}+m$,  
  which implies  $ \frac{D_{\ell+1}(r)}{D_\ell(r)}>1$ for all $\ell\in [1..2m+1]$.
  It  follows  that $D_\ell(r)\leq D_{2m+1}(r)$ for all $\ell\leq 2m+1$.

 \textbf{Case} $r\in[\frac{n}{2}-m.. \frac{n}{2}+m]$.
 
  In this case   $ \frac{D_{\ell+1}(r)}{D_\ell(r)}>1$ for $ \ell\leq \frac{n}{2}+m-r$ and 
    $ \frac{D_{\ell+1}(r)}{D_\ell(r)}<1$ for $\ell >\frac{n}{2}+m-r $,
  from which we conclude $D_\ell(r)\leq D_{\frac{n}{2}+m-r+1}(r) $.

  \textbf{Case} $r\in[\frac{n}{2}+m+1..n-1]$.
  
 In this case $\ell+r>r>\frac{n}{2}+m$, therefore $D_{\ell+1}(r)<D_\ell(r)$.
 Since $\ell$ is at least   $1$,
 we have $D_\ell(r)\leq D_{1}(r)$.
 
 In short, we have  established that
\begin{equation*}
D_\ell(r) \leq
\left\{\begin{array}{ll}D_{2m+1}(r), &    r\in[1.. \frac{n}{2}-m-1],
\\D_{\frac{n}{2}+m-r+1} (r)   ,& r\in[\frac{n}{2}-m.. \frac{n}{2}+m], 
\\ D_1 (r) ,& r\in[\frac{n}{2}+m+1..n-1].
\end{array}\right.
\end{equation*}
  By the definition of $D_\ell(r)$, we have
  $$
  D_{2m+1}(r)={n-2m-2 \choose r-1 }{n\choose r}^{-1},
  $$
  $$
  D_{\frac{n}{2}+m-r+1}(r)= \frac{(n-r)   \cdots (\frac{n}{2}-m)\cdot r  \cdots (\frac{n}{2}-m)}{n \cdots   (n-2m-1)},
  $$
  $$
  D_{1}(r)=\frac{(n-r) \cdot r  (r-1)  \cdots (r-2m)}{n  \cdots   (n-2m-1)}.
  $$
We  observe that $D_{2m+1}(r)$ is in fact the $D_r$ that appeared in (\ref{eq11})  of Lemma~ \ref{case2l}.
Since we have proven  $D_r\leq \frac{1}{n}$ for $r\geq 1$ in  (\ref{par1}) of Lemma~ \ref{case2l},
the same estimate holds for $D_{2m+1}(r)$,
i.e., $D_{2m+1}(r)\leq \frac{1}{n}$ for $r\geq 1$.

For $D_1(r)$, $r\in[\frac{n}{2}+m..n-1]$,
a straightforward calculation shows
$$
\frac{D_1(r+1)}{D_1(r)}=\frac{2mn+n-1-(2m+2)r}{(r-2m)(n-r)}+1.
$$
Thus $D_1(r)<D_1(r+1)$   when $r<r^*\vcentcolon=\frac{2mn+n-1}{2m+2}$ and  $D_1(r)>D_1(r+1)$  when $r>r^*$.
This implies 
\begin{align*}
D_1(r)&\leq D_1(\ceil{r^*})=\frac{n-\ceil{r^*}}{n}  \frac{\ceil{r^*}}{n-1}  \cdots   \frac{\ceil{r^*}-2m}{n-2m-1} \\
&<\frac{n-{r^*}}{n} =\frac{1}{2m+2}\left(1+\frac{1}{n}\right)\leq \frac{1}{\frac{n}{2}+2}\left( 1+\frac{1}{n}\right)\leq \frac{2}{n},
\end{align*}
where in the second last inequality we used the hypothesis $m>\frac{n}{4}$.

For $D_{\frac{n}{2}+m-r+1}(r)$, $r\in[\frac{n}{2}-m.. \frac{n}{2}+m]$,  we have $$\frac{D_{\frac{n}{2}+m-r}(r+1)}{D_{\frac{n}{2}+m-r+1}(r)}=\frac{2r-n+1}{n-r}+1.$$
Hence $D_{\frac{n}{2}+m-r+1}(r)$ decreases with respect to $r$ when $\frac{n}{2}-m\leq r\leq \frac{n}{2}-1$ and increases when $\frac{n}{2}\leq r\leq \frac{n}{2}+m$.
Therefore we have for $r\in[\frac{n}{2}-m.. \frac{n}{2}+m]$
$$
D_{\frac{n}{2}+m-r+1}(r)\leq \max \left\{ D_{2m+1}\left( \tfrac{n}{2}-m\right), D_{1}\left( \tfrac{n}{2}+m\right) \right\}.
$$
Since it has already been shown that both $D_{2m+1}\left( r\right)$, $r\geq 1 $, and $D_1\left( r\right) $, $r\in[\frac{n}{2}+m..n-1]$, are bounded by $\frac{2}{n}$, the same bound holds for $D_{\frac{n}{2}+m-r+1}(r)$, $ r\in[\frac{n}{2}-m.. \frac{n}{2}+m]$.
We thus conclude 
\[
\Pr\left[\lo(V(x))\ge 2m+2\right]=D_\ell(r)\leq \tfrac{2}{n}. 
\qedhere
\]
 \end{proof}
Let  $\mathbf{1}$ denote $(1,\dots,1) \in \{0,1\}^n$. Following the standard notation, we write $\ceil*{i}_2 := \{k\in2\Z\mid k\geq i\}$ for all $i \in \R$. We recall that $\HLB_2(x)$ simply denotes the number of blocks equal to $11$ counted from left to right. In the next lemma we will give a lower bound of $\Omega(n^2)$ on the unary unbiased black-box complexity of the $\DLB$ problem by considering the  potential
\begin{align}\label{eq14}
h_{t}=\min\left\{2\ceil*{\frac{n}{3}}, {}\max_{ s\in[0..t]}\left\{\ceil*{\frac{n+1}{2}}_2,{}\HLB_2(x^{(s)}),\HLB_2(\mathbf{1} -x^{(s)})
\right\}\right\},
\end{align}
that is, the maximum of the $\HLB_2$ value of $x^{(s)}$, $s\in[0..t]$, and the $\HLB_2$ value of $\mathbf{1}-x^{(s)}$,  $s\in[0..t]$, capped into the interval $\left[\ceil*{\frac{n+1}{2}}_2, 2\ceil*{\frac{n}{3}}\right]$.
The preceding lemma enables us to show that for any unary unbiased black-box algorithm, the expected gain  in potential (\ref{eq14}) at each iteration  is $O(\frac{1}{n})$.
The additive drift theorem will then imply  the  $\Omega(n^2)$ lower bound. Before conducting this proof, let us quickly derive a statement similar to Lemma~\ref{indep}, namely that all bits with an index higher than $2m+2$ are uniformly distributed when $h_t$ is at most $m$.

\begin{lemma}\label{lem:indep}
  Consider a run of a unary unbiased black-box algorithm on the \DLB problem. Let $x^{(0)}, x^{(1)}, \dots$ be the search points generated in a run of this algorithm, following the notation of Algorithm~\ref{alg:unbiased}. Conditional on $h_t \le m$, we have that for each $s \in [0..t]$ the random variables $x^{(s)}_i, i \in [2m+3..n]$ are independent and uniformly distributed in $\{0,1\}$. 
\end{lemma}

We omit a formal proof of this statement since it is very similar to the proof of Lemma~\ref{lem4}. The main argument is that the event $h_t$ implies that no search point ever sampled up to iteration~$t$ has the $(m+1)$-st block completed. Consequently, bits with index higher than $2m+2$ have no influence on the run of the algorithm up to this point. 

\begin{lemma}\label{Una_Unb_Low_bdd}
The unary unbiased black-box complexity of the $\DLB$ problem is $\Omega(n^2)$.
\end{lemma}

\begin{proof}
Consider any unary unbiased black-box algorithm,
let $x^{(t)}$ be the search point at time $t$. 
Let $T\vcentcolon = \min\{t\in\mathbb{N}\mid \max\{\HLB_2(x^{(t)}),\HLB_2(\mathbf{1} -x^{(t)})\}\geq 2\ceil*{\frac{n}{3}}\}$ be the first time the $\HLB_2$ value of a search point or its complement  is at least $2\ceil*{\frac{n}{3}}$.
We will use the additive drift theorem (Theorem \ref{adddrift}) 
to prove $E[T ]= \Omega(n^2)$,
from which   the claim follows since $T$ is a  trivial  lower bound for   the   runtime.

Let $t<T$ and let $\mathcal{C}_m$ be the event $\{ h_t=2m\}$ for  $m\in[\frac{1}{2}\ceil*{\frac{n+1}{2}}_2..\ceil*{\frac{n}{3}}-1] $. By definition of $h_t$ and $T$, one of the events $\mathcal{C}_m$ holds.
We  now estimate $E[h_{t+1}-h_t\mid  \mathcal{C}_m]$.
Let $x^{(s)}$, $s\in [0..t]$, be the individual chosen in iteration $t+1$ to generate $x^{(t+1)}$ (see line~4 of Algorithm~\ref{alg:unbiased}, note that here we have~$k=1$).
Let~ $\ell\vcentcolon=|\{i\in[1..2m+2]\mid x^{(s)}_i=1\}|$ be the number of the ones in the leading $2m+2$ bit positions of~ $x^{(s)}$.
Then $1\leq \ell\leq 2m+1$ because otherwise $h_t$ would be not less than  $2m+2$.

If $h_{t+1}>h_t$, then 
either $x^{(t+1)}$ has no less than $2m+2$ leading ones or it has no less than  $2m+2$
leading zeros, i.e., 
\begin{align*}
    \{&h_{t+1}>h_t\}\cap \mathcal{C}_{m}\\
    &=\left(\{\HLB_2(x^{(t+1)})>2m\}\cap \mathcal{C}_{m}\right)\dot\cup\left(\{\HLB_2(\mathbf{1}-x^{(t+1)})>2m\}\cap \mathcal{C}_{m}\right).
\end{align*}
Lemma \ref{calcul}  shows that the first scenario  happens with probability at most~ $\frac{2}{n}$, that is,
$$
\Pr[\HLB_2(x^{(t+1)})>2m\mid \mathcal{C}_{m}]\leq \tfrac{2}{n}.
$$
We recall that $x^{(t+1)}$   is obtained from some $x^{(s)}$, $s\in[0..t]$, via a unary unbiased operator.  By Lemma~\ref{lem:indep}, the $x^{(s)}_i$, $i\in[2m+3..n]$, are all independent and uniformly distributed in $\{0,1\}$ because of the condition $\mathcal{C}_m=\{h_t=2m\}$. 
If we  condition further on the event  $\{\HLB_2(x^{(t+1)})>2m\}\cap \mathcal{C}_{m}$, then $x^{(t+1)}_i$, $i\in[2m+3..n]$, are all independent and uniformly distributed in $\{0,1\}$ and $x^{(t+1)}_i=1$ for  $i\in[1..2m+2]$.
Hence Lemma \ref{upperbrand} can be applied to $x^{(t+1)}$,  yielding 
\begin{align*}
    E[h_{t+1}&-h_t\mid \HLB_2(x^{(t+1)})>2m, \, \mathcal{C}_{m}]\\
    \leq{}&E[\HLB_2(x^{(t+1)})-2m\mid \HLB_2(x^{(t+1)})>2m, \, \mathcal{C}_{m}] \\
    \leq{}& 2m+4-2m=4.
\end{align*}
For the second scenario,  we consider $\mathbf{1}-x^{(s)}$ and $\mathbf{1}-x^{(t+1)}$   instead of $ x^{(s)}$ and $x^{(t+1)}$ to obtain
$$
\Pr[\HLB_2(\mathbf{1}-x^{(t+1)})>2m\mid\mathcal{C}_{m}]\leq \tfrac{2}{n}
$$
and
$$
E[h_{t+1}-h_t\mid \HLB_2(\mathbf{1}-x^{(t+1)})>2m, \, \mathcal{C}_{m}] \leq 4.
$$
Combining the two scenarios, we have 
\begin{align*}
E[h_{t+1}-h_t\mid  \mathcal{C}_{m}]&=\Pr[h_{t+1}>h_t\mid\mathcal{C}_{m}]E[h_{t+1}-h_t\mid  h_{t+1}>h_t,\mathcal{C}_{m}]\\
&\leq{} 2\cdot \tfrac{ 2}{n}\cdot 4=\tfrac{16}{n}.
\end{align*} 
Now we examine the expected initial value of the potential.
First we observe that $\Pr\left[\max\{\HLB_2(x^{(0)}),\HLB_2(\mathbf{1} -x^{(0)})\}\geq \frac{n}{2}\right]=O(2^{-\frac{n}{2}})$ since $x^{(0)}$ is generated uniformly at random in $\{0,1\}^n$.  Thus
 $$E[h_0 ]\leq  \ceil*{\frac{n+1}{2}}_2+ O(n2^{-\frac{n}{2}}) =\frac{n}{2}+O(1).$$
The additive drift theorem (Theorem \ref{adddrift}) thus implies 
\begin{equation*}
E\left[T\right]\geq \frac{2n/3-n/2-O(1)}{16/n}=\Theta(n^2).
\qedhere
\end{equation*}
\end{proof}

Combining Lemma \ref{Una_Unbia_upper_bound} with Lemma \ref{Una_Unb_Low_bdd}, we have proven the following theorem.
\begin{theorem}\label{Una_Unbia}
The unary unbiased black-box complexity of the $\DLB$ problem is $\Theta(n^2)$.
\end{theorem}

\section{Beyond Unary Unbiased Algorithms}
\label{sec:beyond}
We recall that in this work we are generally looking for unbiased algorithms as this is most natural when trying to solve a novel problem without much problem-specific understanding. In Sections \ref{sec1} and \ref{sec:nonel}, we have discussed unary unbiased algorithms. Our theory-guided approach has suggested the Metropolis algorithm, which with a $\Theta(n^2)$ runtime improved over the $O(n^3)$ runtime guarantee previously shown for various classic unbiased evolutionary algorithms. Our $\Omega(n^2)$ lower bound for all unary unbiased algorithms, however, also shows that further improvements are not possible with unary unbiased algorithms, and raises the question if algorithms of higher arity are more powerful for the \DLB problem. 

Following our theory-guided approach, we first exhibit in Section~\ref{5.1} that the binary unbiased black-box complexity of the \DLB problem is $O(n \log n)$. We did not find a natural binary unbiased black-box algorithm for which we could show an $o(n^2)$ runtime, but by resorting to unrestricted arities, which allows for estimation-of-distribution algorithms, we detected that the \emph{significance-based compact genetic algorithm} (sig-cGA)~\cite{DoerrK20tec} has a runtime of $O(n\log n)$ with high probability (Section \ref{5.2}).

\subsection{Unbiased Black-Box Algorithms of Higher  Arity}\label{5.1}

In this section we will exhibit a binary unbiased black-box algorithm (Algorithm \ref{bina}) whose runtime on the $\DLB$ problem is $O(n\log n)$. Our algorithm builds on two ideas used in \cite[Section~5]{DoerrJKLWW11} to design an $O(n \log n)$ binary unbiased black-box algorithm for the $\lo$ function: (i)~We use a pair $(x,y)$ of search points such that the bits $i$ with $x_i = y_i$ are exactly the ones for which we know that the optimum also has this bit value in this bit position. (ii)~We conduct a random binary search in $\{i \in [1..n] \mid x_i \neq y_i\}$ to find in logarithmic time a bit position such that flipping this bit in either $x$ or $y$ (which one will be visible from the fitness) sets it to the value the optimum has in this position. This operation increases the number of correctly detected bits by one. Different from the situation in~\cite{DoerrJKLWW11}, in our setting  such a good bit value cannot always be detected from a fitness gain (due to the deceptive nature of the \DLB problem). We overcome this by regarding the $\HLB_1$ function instead. Note that there is a (non-monotonic) one-to-one relation between the fitness values of $\DLB$ and $\HLB_1$. Hence a black-box algorithm for \DLB has access also to the $\HLB_1$ value of any bit string.

We first recall from~\cite{DoerrJKLWW11} the two binary unbiased variation operators $\rwd(\cdot,\cdot)$ and $\sido(\cdot,\cdot)$. The operator $$\rwd(\cdot,\cdot)$$ takes two bit strings $y$ and $y'$ as input and outputs a bit string $z$ of which the components $z_i$, $i\in [1..n]$, are determined by the following rule: we have $z_i=y_i$ if $y_i=y'_i$, and $z_i$ is chosen uniformly at random from $\{0,1\}$ if $y_i\neq y'_i$. This operator is also (and better) known as uniform crossover, but to ease the comparison with~\cite{DoerrJKLWW11} we keep the notation used there.
The operator $\sido(\cdot,\cdot)$ takes two bit strings $y$ and $y'$ as input
and outputs $y'$ if the Hamming distance between $y$ and $y'$ is one, otherwise it outputs   $y$.

We now state our black-box algorithm (Algorithm \ref{bina}). Since, as discussed, any black-box algorithm for the \DLB problem also has access to the $\HLB_1$ values, we use this function as well in the algorithm description. Algorithm~\ref{bina} is initialized by generating a random search point $x$ and its complement $y$ (we note that taking the complement is the unary unbiased operator of flipping $n$ bits).


In each iteration of the \textbf{for} loop, 
we  ensure   $\HLB_1(y)\leq \HLB_1(x)$ by exchanging $x$ and~ $y$ when necessary and run the subroutine in the 
lines~\ref{line1}-\ref{line2}, which flips one bit in $y$ in a way that increases the $\HLB_1$ value of $y$.
To be more precise, 
the subroutine does a random binary search for such a  bit position.
At the beginning of the subroutine, $y'$ is set to be $x$.
In each iteration of the \textbf{repeat} loop,
$y''$ is sampled from  $\rwd(y,y')$.
If  $\HLB_1(y'')>\HLB_1(y)$,
then  we accept $y''$ and attribute its value to  $y'$, otherwise we refuse $y''$  and proceed to the next iteration of the subroutine.

This subroutine terminates when $y'$ and $y$ differ only in one bit position. At this moment  $y$ is set to be equal to $y'$. Now $x$ and $y$ have exactly one more bit in common, which means the number of correctly detected bits is increased by one. By induction, after $i$ iterations of the \textbf{for} loop, $x$ and $y$ agree on $i$ bit positions. After $n$ iterations, $x$ and $y$ are both equal to the optimum. 

\begin{algorithm2e}%
\textbf{initialization:} generate a search point $x$ uniformly at random in $\{0,1\}^n$ 
and let $y$ be the complement of  $x$;

	\For{$i\in[1..n]$}{
	\lIf{$\HLB_1(y)>\HLB_1(x)$}{$(x, y) \assign (y, x)$}
	
	$y' \assign x$;
	
	$H\assign\HLB_1(y)$;
	
	\Repeat{$\HLB(y)>H$}{\label{line1}
	$y''\assign \rwd(y,y') $;\label{line6}
	
	\lIf{$\HLB_1(y)<\HLB_1(y'')$}{$y'\assign y''$}\label{line4}
	$y\assign\sido(y,y')$;
	}\label{line2}
	
  }
\caption{A binary unbiased black-box algorithm for maximizing the $\DLB$ function. }
\label{bina}
\end{algorithm2e}

To analyze the time complexity of Algorithm \ref{bina}, 
we first consider the runtime of the \textbf{repeat}  subroutine.
\begin{lemma}\label{lem17}
The subroutine in lines~\ref{line1}-\ref{line2} of Algorithm \ref{bina} finishes in an expected runtime of $O(\log n)$.
\end{lemma}
\begin{proof}
For   sake of simplicity we suppose $\lo(y)=2m+1$ for some $m$ at the beginning of the   subroutine,
the cases $\DLB(y)=2m+1$ and $\DLB(y)=\lo(y)=2m$ can be treated similarly.

Since  $\HLB_1(y'')>2m+1=\HLB_1(y)$   if and only if $y''_{2m+2}=1$   and this second event happens with probability $\frac{1}{2}$ (line \ref{line6}),
the value of $y''$ is assigned to $y'$ with probability $\frac{1}{2}$.

Now we consider the bit positions, other than the bit position $2m+2$, on which $y$ and $y'$ are different. Since the bit values of $y''$ on these positions are chosen at random, the probability that at least half of them coincide with corresponding bit values in $y$ is at least $\frac{1}{2}$ by virtue of symmetry.
Therefore, 
at the end of an iteration of the subroutine,
with probability $\frac{1}{4}$
the number of 
bit positions other than the bit position $2m+2$, on which~ $y$ and $y'$ are different, is at most half as before. We note that the number of different bits never increases. Hence,
we conclude that in expected runtime $O(\log n)$, 
$y$ and $y'$ are only different in bit position $2m+2$.
\end{proof}

As discussed above,
$x$ and $y$ are set to the optimum  after having run the \textbf{for} loop for  $n$ times. Lemma \ref{lem17} thus implies the following theorem on  the binary unbiased black-box complexity of the $\DLB$ problem.
\begin{theorem}
The binary unbiased black-box complexity of the $\DLB$ problem is $O(n\log n)$.
\end{theorem}

Algorithm \ref{bina}  reveals that the ability to learn how good solutions look like plays an important role in solving the $\DLB$ problem. This inspires the study of EDAs, and more specifically, the  significance-based EDA proposed by Doerr and Krejca~\cite{DoerrK20tec} in the following.

\subsection{Significance-Based Compact Genetic Algorithm (sig-cGA) }\label{5.2}
 Estimation-of-distribution algorithms (EDAs) optimize a function $f:\{0,1\}^n\rightarrow \mathbb{R}$ by evolving a probabilistic model of the solution space $\{0,1\}^n$ in such a direction that this probabilistic model becomes more likely to sample the optimum.

We now present the   \emph{sig-cGA} (Algorithm \ref{sigcga}) proposed by Doerr and Krejca \cite{DoerrK20tec}. To this end, we first define the notion of \textit{frequency vectors}. A frequency vector $\tau = (\tau_i)_{i\in[1..n]}$ is a vector whose components represent a probability, i.e., $\tau_i\in[0,1]$, $ i\in[1..n]$. A probability distribution on $ \{0,1\}^n$  is associated to a frequency vector $\tau$ in the following way:
 an individual $x\in\{0,1\}^n$ follows this probability distribution if and only if each component   $x_i$ of $x$ independently follows the Bernoulli distribution of parameter $\tau_i$.

The sig-cGA utilizes   frequency vectors whose components take values in $\{\frac{1}{n},\frac{1}{2},1-\frac{1}{n}\}$ to represent a probability  distribution on $\{0,1\}^n$ in the above way. For each bit position $i\in[1..n]$, the sig-cGA keeps a history $H_i\in\{0,1\}^*$   for significance inferring that will be explained in detail in the following.
In iteration $t$, two individuals,  $x$ and $y$,  are independently sampled from the probability distribution associated with $\tau^{(t)}$.  They are then  evaluated under $f$, the one with better fitness is called the winner and is denoted by $z$. In the case of a tie    the winner is chosen at random.  Now for each $i\in[1..n]$, the value of $z_i$ is added to the history $H_i$. If a statistical significance of $1$ ($0$ resp.) is detected in $H_i$, then $\tau_i^{(t+1)}$ is set to $1-\frac{1}{n}$  ($\frac{1}{n}$ resp.) and the history $H_i$ is emptied.  
To be more precise, we define in the following the function $\operatorname{sig}_{\varepsilon}$ taking values in $\{\textsc{up},\textsc{down},\textsc{stay}\}$ and we say that a significance of $1$ ($0$ resp.)  is detected when $\operatorname{sig}_{\varepsilon}\left(\tau_{i}^{(t)}, H_{i}\right)= \textsc{up}$ (\textsc{down} resp.).

For all $\varepsilon, \mu\in\mathbb{R}^+$, let $s(\varepsilon, \mu)=\varepsilon \max \{\sqrt{\mu \log n}, \log n\}$. For all $H\in\{0,1\}^*$, let $H[k]$ be the string of its last $k$ bits and let $\|H[k]\|_0$ ($\|H[k]\|_1$ resp.) denote the number of zeros (resp. ones) in $H$.  Then for all $p\in\{\frac{1}{n},\frac{1}{2},1-\frac{1}{n}\}$ and  $H\in\{0,1\}^*$, $\operatorname{sig}_{\varepsilon}(p, H)$ is defined by 
$$  
\operatorname{sig}_{\varepsilon}(p, H)=
\begin{cases}
  \textsc{up} &\text{if } p \in\left\{\frac{1}{n}, \frac{1}{2}\right\} \wedge \exists m \in \mathbb{N}:\\
 & \left\|H\left[2^{m}\right]\right\|_{1} \geq 2^{m} p+s\left(\varepsilon, 2^{m} p\right),\\
  \textsc{down} &\text{if } p \in\left\{ \frac{1}{2}, 1-\frac{1}{n}\right\} \wedge \exists m \in \mathbb{N}:\\
  &\left\|H\left[2^{m}\right]\right\|_{0} \geq 2^{m} (1-p
  )+s\left(\varepsilon, 2^{m} (1-p)\right),\\
  \textsc{stay} &\text{else}.
  \end{cases}
$$
\begin{algorithm2e}
	$t \assign 0$;	
	
	\lFor{$i\in[1..n]$}{$\tau_{i}^{(0)} \leftarrow \frac{1}{2} \text { and } H_{i} \leftarrow \emptyset$}
	
	\Repeat{termination criterion met}{ 
	$x, y \leftarrow \text { offspring sampled with respect to } \tau^{(t)}$;
	
	$z \leftarrow \text { winner of } x \text { and } y \text { with respect to } f$, chosen at random in case of a tie;
	
	\For{$i\in[1..n]$}{
	$H_{i} \leftarrow H_{i} \circ z_{i}$;
	
	\lIf{$\operatorname{sig}_{\varepsilon}\left(\tau_{i}^{(t)}, H_{i}\right)= \textsc{up}$}{$\tau_{i}^{(t+1)} \leftarrow 1-\tfrac1n$}\label{lin3}
	\lElseIf{$\operatorname{sig}_{\varepsilon}\left(\tau_{i}^{(t)}, H_{i}\right)=\textsc{down}$}{$\tau_{i}^{(t+1)} \leftarrow \tfrac1n$}{}
	\lElse{$\tau_{i}^{(t+1)} \leftarrow \tau_{i}^{(t)} $}
	\lIf{$\tau_{i}^{(t+1)} \neq \tau_{i}^{(t)}$}{$H_{i} \leftarrow \emptyset$}\label{lin4}
	$t\assign t+1$;\label{lin2}
	}
	}
\caption{The sig-cGA with parameter $\varepsilon$ and significance function $\operatorname{sig}_\varepsilon$ optimizing $f:\{0,1\}^n\rightarrow \mathbb{R}$. By $\circ$ we denote the concatenation of two strings (here only used for appending a single letter to a string).}
\label{sigcga}
\end{algorithm2e}

The following lemma (Lemma 2 in \cite{DoerrK20tec}) shows that the sig-cGA, with  probability at least $1-n^{-\varepsilon/3}\log_2 k$, does not detect a significance at a position with no bias in selection, that is, with a high probability it does not detect  a false significance.
\begin{lemma}[Lemma~2~in~\cite{DoerrK20tec}]\label{DoeK}
Consider the sig-cGA   with  $\varepsilon\geq1$. Let  $i\in[1..n]$ be a bit position  and suppose that the distribution of 1s in  $H_i$ follows a binomial law with $k$ tries and success probability $\tau_i$. Then $$\Pr\left[\operatorname{sig}_{\varepsilon}\left(\tau_{i}^{(t)}, H_{i}\right) \ne \textsc{stay}\right] \le n^{-\varepsilon/3}\log_2 k,$$ that is, $\tau_i$ changes with a probability  of at most   $n^{-\varepsilon/3}\log_2 k$.
\end{lemma}
The preceding lemma readily  implies  the following corollary.
\begin{corollary}\label{DoeKcor}
Consider the sig-cGA   with  $\varepsilon\geq1$. Let  $i\in[1..n]$ be a bit position  and suppose that the distribution of 1s in  $H_i$ follows a binomial law with $k$ tries and a success probability of at least $\tau_i$. Then $\tau_i$ decreases in an iteration with a probability  of at most   $n^{-\varepsilon/3}\log_2 k$.
\end{corollary}

\subsubsection{The Sig-cGA Solves the DLB Problem in $O(n\log n)$ Time }\label{5.3}
Now we show that with   probability at least $1-O\left(n^{2-\varepsilon / 3} \log ^{2} n\right)$, the sig-cGA samples the optimum of the $\DLB$ function within $O(n\log n)$ fitness evaluations (Theorem \ref{sigtime}).
First we show that when no  $\tau_{i}$ is set to $\frac{1}{n}$, the probability of adding a 1 to the   history $H_i$ in an iteration is at least $\tau_i$, which will allow us to use Corollary~\ref{DoeKcor}.

\begin{lemma}\label{lem19}
Let $m\in[1..\frac{n}{2}-1]$. Let $(\tau_i)_{i\in[1..n]}\in\{\frac{1}{n},\frac{1}{2},1-\frac{1}{n}\}^n$ be a frequency vector. Consider one iteration of the sig-cGA optimizing the $\DLB$ function. Let $x$ and $y$ be the two individuals sampled according to $(\tau_i)_{i\in[1..n]}$. Then, conditioning on $\{\min\{\DLB(x), \DLB(y)\}\geq 2m\}$, a 1 is saved in $H_{2m+1}$ with probability
\begin{equation}\label{probaincre}
\begin{split}
p(\tau_{2m+1},&\tau_{2m+2})\\
&\vcentcolon=\tau_{2m+1}\left(-\tau_{2m+2}^2+3\tau_{2m+2}+(\tau_{2m+2}^2-3 \tau_{2m+2}+1)\tau_{2m+1} \right).
\end{split}
\end{equation}
If $\tau_{2m+1},\tau_{2m+2}\in\{\frac{1}{2},1-\frac{1}{n}\}$, then the above term is bounded from below by   $\tau_{2m+1}$.
If further $\tau_{2m+1}=\frac{1}{2}$ and $\tau_{2m+2}\in\{\frac{1}{2},1-\frac{1}{n}\}$, then the above term is bounded from below by   $\frac{9}{16}$.
\end{lemma}

\begin{proof}
Under the condition $\{\min\{\DLB(x), \DLB(y)\}\geq 2m\}$, the bit value saved in $H_{2m+1}$ is completely determined by the  bits $x_{2m+1}$, $x_{2m+2}$, $y_{2m+1}$, and $y_{2m+2}$. 
In the following we  assume $m=0$ for the ease of presentation, but the general result can be obtained in exactly the same way.
We calculate 
\begin{align*}
    \Pr[&\text{1 is saved in } H_{1}]\\
    ={}& \Pr[x_{1}=x_{2}=1]\\ &+\Pr[x_{1}=1,x_{2}=0]\left(\Pr[y_{1}=1]+\tfrac{1}{2}\Pr[y_{1}=0,y_{2}=1]\right)\\
    &+\Pr[x_{1}=0,x_{2}=1]\left(\Pr[y_{1}=y_{2}=1]+\tfrac{1}{2}\Pr[y_{1}=1,y_{2}=0]\right)\\
    &+\Pr[x_{1}=0,x_{2}=0]\Pr[y_{1}=y_{2}=1] \\
    ={}&\tau_{1}\tau_{2}+\tau_{1}\left(1-\tau_{2}\right)\left(\tau_{1}+\tfrac{1}{2}\left(1-\tau_{1}\right)\tau_{2}\right)\\
    &+\left(1-\tau_{1}\right)\tau_{2}\left(\tau_{1}\tau_{2}+\tfrac{1}{2}\tau_{1}\left(1-\tau_{2}\right)\right) \\
    &+\left(1-\tau_{1}\right)\left(1-\tau_{2}\right)\tau_{1}\tau_{2}\\
    ={}&\tau_{1}\left(-\tau_{2}^2+3\tau_{2}+\left(\tau_{2}^2-3 \tau_{2}+1\right)\tau_{1} \right).
\end{align*}
Now if $\tau_{1},\tau_{2}\in\{\frac{1}{2},1-\frac{1}{n}\}$, then 
$$-\tau_{2}^2+3\tau_{2}
    = -(\tau_{2}-\tfrac{3}{2})^2+\tfrac{9}{4}\geq-(\tfrac{1}{2}-\tfrac{3}{2})^2+\tfrac{9}{4}=\tfrac{5}{4}.$$
Thus we have $\tau_{2}^2-3 \tau_{2}+1 \leq -\tfrac{1}{4}<0$ and 
\begin{align*}
   \frac{p(\tau_{1},\tau_{2})}{\tau_{1}}&= -\tau_{2}^2+3\tau_{2}+(\tau_{2}^2-3 \tau_{2}+1)\tau_{1}\\
    &\geq -\tau_{2}^2+3\tau_{2}+(\tau_{2}^2-3\tau_{2}+1)\left(1-\tfrac{1}{n}\right)\\
    &=  -\tfrac{1}{n}(\tau_{2}^2-3\tau_{2})+1-\tfrac{1}{n}\\
    &\geq -\tfrac{5}{4n}+1-\tfrac{1}{n}=1+\tfrac{1}{4n}>1.
\end{align*}
Supposing further that $\tau_{1}=\frac{1}{2}$, we have 
\begin{align*}
    \frac{p(\tau_{1},\tau_{2})}{\tau_{1}} &= \tfrac{1}{2}(-\tau_{2}^2+3\tau_{2}+1) 
    \geq \tfrac{1}{2}(\tfrac{5}{4}+1) = \tfrac{9}{8}.
    \qedhere
\end{align*}
\end{proof}

\begin{corollary}\label{coro20}
If in an iteration of the sig-cGA, the frequency vector is in $\{ \frac{1}{2},1-\frac{1}{n}\}^n$, then for any $i\in[1..n]$, a 1 is saved in $H_{i}$ with probability at least $\tau_{i}$.
\end{corollary}

\begin{proof}
Let $x$ and $y$ be the two individuals sampled in this iteration. Let $i=2m+1$ for some $m\in [1..\tfrac{n}{2}-1]$.
Then the preceding lemma shows that conditioning on $\{\min\{\DLB(x), \DLB(y)\}\geq 2m\}$,
a 1 is saved in $H_{2m+1}$ with probability  at least $\tau_{2m+1}$.
Under the condition $\{\min\{\DLB(x), \DLB(y)\}< 2m\}$, 
the probability that a 1 is saved in $H_{2m+1}$ is equal to the probability $\tau_{2m+1}$ that a 1 is sampled in this position, because bit $2m+1$ is not relevant for the selection. Combining the two cases, the claim follows for $i=2m+1$. The   symmetry between $i=2m+1$ and $i=2m+2$   concludes the proof.
\end{proof}
\begin{lemma}\label{freqdecr}
Consider the sig-cGA with $\eps>3$.
The probability that during the first $k$ iterations at least one  frequency decreases  is at most $kn^{1-\varepsilon/3}\log_2 k$.
\end{lemma}
\begin{proof}
Consider the event that in the first $t$ iterations, no frequency has ever been decreased. Denote its probability by $p^{(t)}$. 

Conditioning on this event, Corollary~\ref{coro20} can be applied to verify  the hypothesis of 
  Corollary~\ref{DoeKcor}, which implies 
that the conditional probability that no frequency decreases in the $(t+1)$-th iteration is at least 
$1-n^{1-\eps/3}\log_2 k$.  Therefore $p^{(t+1)}\geq p^{(t)}(1-n^{1-\eps/3}\log_2 k)$. By induction, $p^{(k)}\geq (1-n^{1-\eps/3}\log_2 k)^{k}\geq 1-kn^{1-\eps/3}\log_2 k$.
\end{proof}

\begin{theorem}\label{sigtime}
The runtime of the sig-cGA  with $\eps>6$ on $\DLB$ is $O(n\log n)$ with   probability at least $1-O\left(n^{2-\varepsilon / 3} \log ^{2} n\right)$.
\end{theorem}
\begin{proof}
By taking $k=O(n\log n)$ in Lemma \ref{freqdecr}, we obtain that with probability   at least $1-O\left(n^{2-\varepsilon / 3} \log ^{2} n\right)$, no frequency decreases in the first $O(n \log n)$ iterations.  We condition on this event in what follows.

During the runtime of the sig-cGA, a  block $(x_{2m+1},x_{2m+2})$ is called \emph{critical} if $m$ is such that $\tau_i=1-\frac{1}{n}$ for all $i\in[1..2m]$,  and that  $\tau_{2m+1}=\frac{1}{2}$ or $\tau_{2m+2}=\frac{1}{2}$. 
Suppose that such a critical block is created with $\tau_{2m+1}=\frac{1}{2}$.  We prove that the history of position $2m+1$   saves 1s significantly more often than 0s and hence  the frequency $\tau_{2m+1}$ is set to $1-\frac{1}{n}$ after $O(\log n)$ iterations. 

Let $G$ denote the event that we save a 1 in $H_{2m+1}$ in one iteration. We now calculate a lower bound of $G$  under the condition  that no frequency is decreased   within first $O(n\log n)$ iterations. 

Let $A$ denote the event that at least one of the two individuals sampled in an iteration has a \DLB value smaller than $2m$. When $A$ happens, a 1 is saved in $H_{2m+1}$ with probability $\tau_{2m+1}=\frac{1}{2}$. If $A$ does not happen, then  the probability of saving a 1 to $H_{2m+1}$ is equal to $p(\tau_{2m+1},\tau_{2m+2})$ defined in (\ref{probaincre}). Hence we can decompose $\Pr[G]$ as
$$
\operatorname{Pr}[G]=\operatorname{Pr}[A]  \tau_{2m+1}+\operatorname{Pr}[\bar{A}]  p(\tau_{2m+1},\tau_{2m+2}).
$$
Since $\bar{A}$ is equivalent to both individuals have 1s in the first $2m$ bit positions,  we have $\Pr[\bar{A}]=\left(\left(1-\frac{1}{n}\right)^{2m}\right)^2\geq \left(\left(1-\frac{1}{n}\right)^{n-2}\right)^2\geq\left( e^{-1}\right)^2=e^{-2}$. With Lemma \ref{lem19}, this implies
\begin{align*}
    \Pr[G]\geq& (1-\Pr[\bar{A}])\tau_{2m+1}+\Pr[\bar{A}]\tfrac{9}{8}\tau_{2m+1} \\
    \geq & \left(1+\tfrac{1}{8}\Pr[\bar{A}]\right)\tau_{2m+1} = (1+\tfrac{1}{8}e^{-2})\tau_{2m+1},
\end{align*}
since we assumed  that $\tau_{2m+1}=\frac{1}{2}$.

With this lower bound on $\Pr[G]$ we bound the probability that a significance of 1s in $H_{2m+1}$ is detected within $k=O(\log n)$ iterations. To this end,
we consider the process $X \sim \operatorname{Bin}\left(k,\left(1+e^{-2} / 8\right) \tau_{2m+1}\right)$, which is stochastically dominated by the actual process of saving 1s at position $H_{2m+1}$. It follows from the definition that
\begin{align*}\operatorname{Pr}&\left[X\leq k \tau_{2m+1}\right.\left.+s\left(\varepsilon, k \tau_{2m+1}\right)\right]  \\
 &= \operatorname{Pr}\left[X \leq E[X]-\left(\tfrac{k}{8} e^{-2} \tau_{2m+1}-s\left(\varepsilon, k \tau_{2m+1}\right)\right)\right]. \end{align*}
For $k>64 e^4\eps^2\tau_{2m+1}^{-1}\log n=\Theta(\log n)$ we have $s\left(\varepsilon, k \tau_{2m+1}\right)=\eps\sqrt{k\tau_{2m+1}\log n}$ and thus $\frac{k}{8} e^{-2} \tau_{2m+1}-s\left(\varepsilon, k \tau_{2m+1}\right)>0$. Let $c\vcentcolon=64 e^4\eps^2\tau_{2m+1}^{-1}$ and consider  $k>4c \log n$ iterations.
We have  
$$s\left(\varepsilon, k \tau_{2m+1}\right)=\eps\sqrt{k\tau_{2m+1}\log n}<\eps k\sqrt{\frac{\tau_{2m+1}}{4c}}=\frac{\tau_{2m+1}k}{16e^2},
$$ 
which implies
$$
\frac{k}{8} e^{-2} \tau_{2m+1}-s\left(\varepsilon, k \tau_{2m+1}\right)>\frac{\tau_{2m+1}k}{16e^2}=\vcentcolon \lambda,
$$
and hence
$$
 \operatorname{Pr}\left[X \leq E[X]-\left(\tfrac{k}{8} e^{-2} \tau_{2m+1}-s\left(\varepsilon, k \tau_{2m+1}\right)\right)\right]
 \leq \Pr\left[X \leq E[X]-\lambda\right].
 $$
 We use a  Chernoff inequality to  calculate an upper bound for this probability. By definition  $\operatorname{Var}[X]=k\left(1+e^{-2} / 8\right) \tau_{2m+1}\left(1-\left(1+e^{-2} / 8\right)\tau_{2m+1}\right)\geq\lambda  $. It is straightforward that 
 \begin{align*} 
 \frac{ \lambda^2}{ \Var [X]} >  \frac{ \lambda^2}{ k\tau_{2m+1}} 
    =    \frac{ k\tau_{2m+1}}{256e^4} >  \frac{ c\tau_{2m+1}\log n}{64e^4} =\eps^2 \log n.
\end{align*}
Now the Chernoff inequality (Theorem \ref{chern}) implies 
\begin{align*}
    \Pr[X& \leq E[X]-\lambda ]\leq \exp\left(\frac{-\lambda^2}{3\Var [X]}\right)< \exp\left(-\frac{\eps^2}{3 }\log n\right) = n^{-\frac{\eps^2}{3}}.  
\end{align*}
We have thus proven that with probability at most $n^{-\frac{\eps^2}{3}}$, $\tau_{2m+1}$ is not set to $1-\tfrac{1}{n}$ after $O(k)=O(4c\log n)=O(\log n)$ iterations. Due to the symmetry between positions $2m+1$ and $2m+2$, the same holds true for bit positions $2m+2$, $m\in[1..\frac{n}{2}-1]$. Hence with probability at most $ n^{1-\frac{\eps^2}{3}}$,  for some $i\in[1..n]$, $ \tau_{i} $  is not set to $1-\frac{1}{n}$ within the first   $O(n \log n)$ iterations, that is, with probability at least  $1-n^{1-\frac{\eps^2}{3}}$, every $\tau_i$, $i\in[1..n]$, is set to $1-\tfrac{1}{n}$ within the first $O(nk)=O(n\log n)$ iterations.

Once this is achieved, the probability  to sample the optimum  $\mathbf{1}=(1,\dots,1) $ from $(\tau_i)_{i\in[1..n]}$ is equal to $\left(1-\frac{1}{n}\right)^n>\frac{1}{2e}$. 
Let $\eps'>\tfrac{\eps}{3}-2$ be a constant.
Since we condition on no frequency dropping, the sig-cGA samples  $\mathbf{1} $ at least once within $\eps'\frac{1}{2}\left(\log \left(\frac{2e}{2e-1}\right)\right)^{-1}\log n =O(\log n )$ iterations with a probability of at least 
$$1-\left(1-\frac{1}{2e}\right)^{\eps'\left(\log \left(\frac{2e}{2e-1}\right)\right)^{-1}\log n}=1-n^{-\eps'}.$$

Recall that   no frequency decreases within the first $O(n \log n)$ iterations  with probability at least $1-O\left(n^{2-\varepsilon / 3} \log ^{2} n\right)$. Combining the preceding results, we have proven that the optimum is sampled  within $O(n \log n) $ iterations
  with probability at least
$$\left(1-O\left(n^{2-\varepsilon / 3} \log ^{2} n\right)\right)\left(1-n^{1-\frac{\eps^2}{3}}\right)(1-n^{-\eps'})=1-O\left(n^{2-\varepsilon / 3} \log ^{2} n\right)$$ 
because of the definition of $\eps'$ and the inequality $2-\frac{\eps}{3}>1-\frac{\eps^2}{3}$.
\end{proof}

A reviewer of this work noted that also the \emph{significant bit voting algorithm} proposed in~\cite{RoweA19} could have an $O(n \log n)$ runtime on the \DLB problem. This algorithm, with population size $\mu \in \N$, works as follows. Initially, all bits are undecided. In each of the $n$ iterations, exactly for one bit a value is decided. This is done via the following procedure. (i)~Independently, $\mu$~individuals are sampled. The undecided bits are sampled uniformly at random, the decided ones deterministically take the decided value. (ii)~From this sample, $\mu/3$ individuals with highest fitness are selected. (iii)~Among these, an undecided bit position $i \in [1..n]$ is selected in which the number of zeros and ones differs maximally. (iv)~The $i$-th bit is declared decided, and this is with the more frequently occurring bit value in the selected subpopulation. 

It is fairly easy to prove that when the population size $\mu$ is at least logarithmic (with sufficiently large implicit constant), then with high probability in a run of this algorithm on \DLB the bits are decided on order of increasing block index, and all decided bits receive the value one. This shows that this algorithm with high probability optimizes \DLB in time $O(n \log n)$.

For two reasons we find the sig-cGA a more convincing example of a natural randomized search heuristic. On the one hand, the sig-cGA surely converges, that is, eventually outputs the right solution. On the other hand, more critically, we feel that the significant bit voting algorithm is overfitted to problems that are easily optimized by sequentially fixing one bit value. In fact, we have strong doubts that this algorithm with a logarithmic population size can optimize the \onemax benchmark with a probability higher than $o(1)$, and in fact, higher than $1/p$ where $p$ is any polynomial in $n$. The reason is that in the first iteration, the selected individuals all have at most $\frac n2 + O(\sqrt{n \log n})$ one-bits. Hence even in the selected population, a fixed bit is one with probability $\frac 12 + O(n^{-1/2} \sqrt{\log n})$ only. Hence with a logarithmic population size $\mu$, we expect to have both $\mu/6 \pm o(1)$ zeroes and ones in a fixed bit position (in the selected subpopulation of size $\mu/3$). This gives little hope that the small advantage of the value one is safely detected. We sketched this argument for the first iteration, but the same argument applies to all other iterations. For this reason, it appears very likely that at least one bit-value is set wrongly in the first $n/2$ iterations. We note that this is not a formal proof since we have not taken into account the (light) stochastic dependencies among the bit positions. We are nevertheless pessimistic that this algorithm has a good performance on problems like \onemax. We note the only mathematical runtime analysis so far for this algorithm considers exclusively the \leadingones problem.

\section{Experiments}
\label{sec:exp}
To see how the algorithms compare on concrete problem sizes, we ran the \oea, UMDA, Metropolis algorithm, and sig-cGA on the DLB function for $n=40,80,\dots,200$. 
We used the standard mutation rate $p = 1/n$ for the \oea, the population sizes $\mu=3n\ln n$ and $\lambda=12\mu$ for the UMDA (as in~\cite{DoerrK21ecj}), and the temperature parameter $\alpha=3$ (greater than $\sqrt 2+1$ as suggested from Theorem~\ref{MA_runtime}) for the Metropolis algorithm. For the sig-cGA, we took $\epsilon=2.5$ since we observed that this was enough to prevent frequencies from moving to an unwanted value, which only happened one time for $n=40$. Being still very slow, for this algorithm we could only perform 10 runs for problem sizes 40, 80, 120, and 160.

Our experiments clearly show an excellent performance of the Metropolis algorithm, which was suggested as an interesting algorithm by our theoretical analysis. The two EDAs, however, perform worse than what the asymptotic results suggest. Such a discrepancy between theoretical predictions and experimental results, stemming from the details hidden in the asymptotic notation, has been observed before and is what triggered the research area of \emph{algorithm engineering}~\cite{MullerS01}. 

\begin{figure}[!ht]
\centering
\includegraphics[width=3.8in]{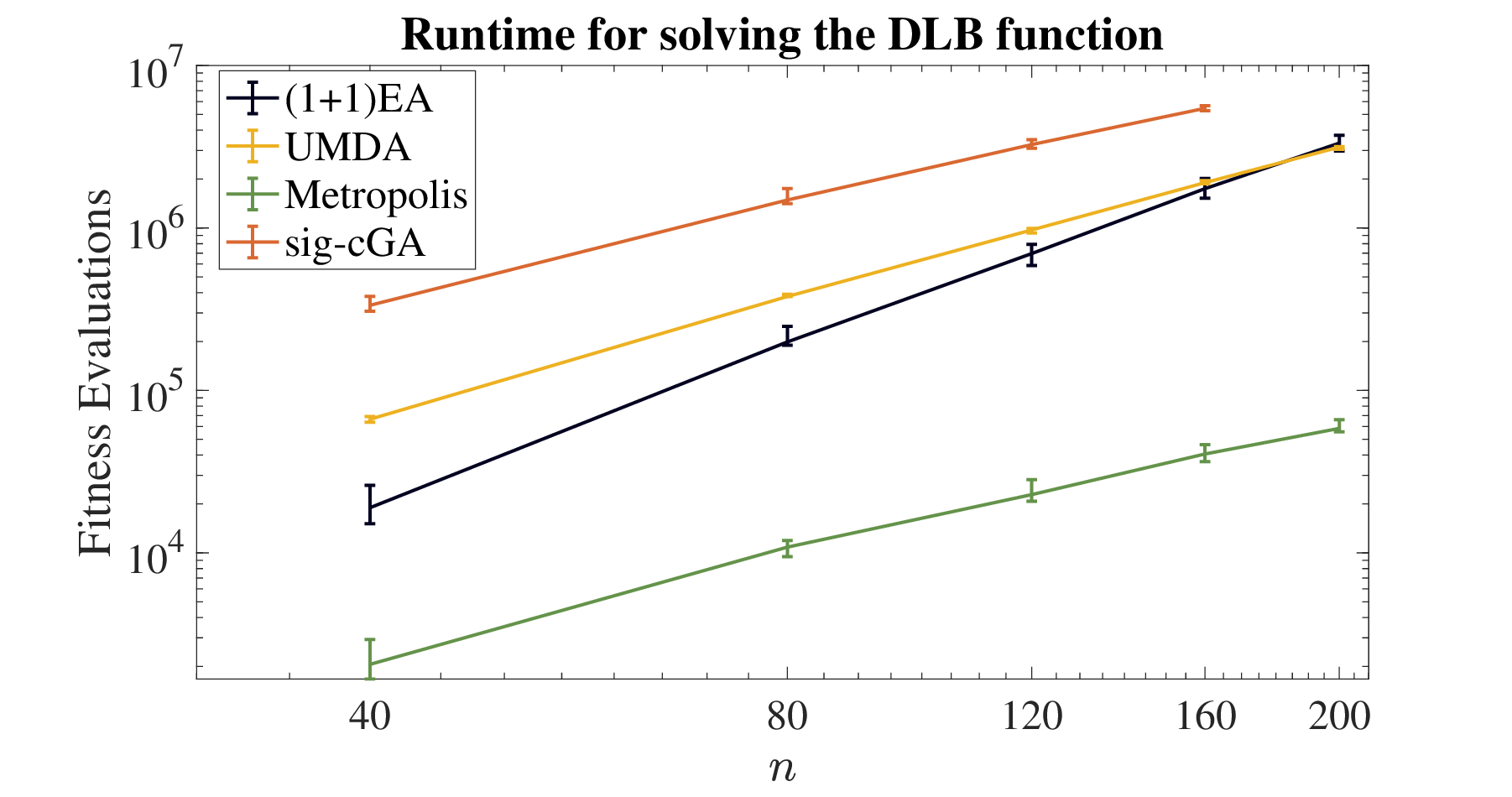}
\caption{The median number of fitness evaluations (with the first and third quartiles) of \oea, UMDA, Metropolis algorithm, and the sig-cGA on DLB with $n\in\{40,80,120,160,200\}$ in 20 independent runs (10 runs for $n\in\{40,80,120,160\}$ 
for the sig-cGA).}
\label{fig:runtime}
\end{figure}

\section{Conclusion and Outlook}
\label{sec:con}
To help choosing an efficient randomized search heuristic when faced with a novel problem, we proposed a theory-guided approach based on black-box complexity arguments and applied it to the recently proposed DLB function. Our approach suggested the Metropolis algorithm, for which little theoretical support before existed. Both a mathematical runtime analysis and our experiments proved it to be significantly superior to all previously analyzed algorithms for the \DLB problem.

We believe that our approach, in principle and in a less rigorous way, can also be followed by researchers and practitioners outside the theory community. Our basic approach of (i)~trying to understand how the theoretically best-possible algorithm for a given problem could look like and then (ii)~using this artificial and problem-specific algorithm as indicator for promising established search heuristics, can also be followed by experimental methods and by non-rigorous intuitive considerations.

\section*{Acknowledgments}
This work was supported by National Natural Science Foundation of China (Grant No. 62306086), Science, Technology and Innovation Commission of Shenzhen Municipality (Grant No. GXWD20220818191018001), Guangdong Basic and Applied Basic Research Foundation (Grant No. 2019A1515110177).

This work was also supported by a public grant as part of the Investissement d'avenir project, reference ANR-11-LABX-0056-LMH, LabEx LMH.
%
%

}
\end{document}